\documentclass[preprint,12pt,authoryear]{elsarticle}

\usepackage{amssymb}
\usepackage{amsmath}

\usepackage[utf8]{inputenc} 
\usepackage[T1]{fontenc}    
\usepackage{hyperref}       
\usepackage{url}            
\usepackage{booktabs}       
\usepackage{amsfonts}       
\usepackage{nicefrac}       
\usepackage{microtype}      
\usepackage{xcolor}         

\usepackage{graphicx}
\usepackage{subfigure}
\usepackage{booktabs} 

\usepackage{hyperref}

\usepackage{algorithm}
\usepackage{algorithmic}

\usepackage{amsmath}
\usepackage{amsfonts}
\usepackage{amssymb}
\usepackage{amsthm}
\usepackage{cuted}
\usepackage{xcolor}
\usepackage{enumitem}

\RequirePackage{caption}
\usepackage{url}

\newcommand\distance{\mathfrak{d}}
\newcommand\distribution{\mathfrak{f}}

\newcommand\measureRho{\mathfrak{m}}

\newcommand{\pr}{\mbox{\sf P}}
\newcommand{\ex}{{\bf\sf E}}               

\newcommand{\cR}{\mathcal{R}}

\newcommand{\calp}{{\cal P}}

\newcommand{\al}{\alpha}                
   
\newcommand{\ep}{\epsilon}

\newcommand{\ra}{\rightarrow}           

\newcommand{\sinc}{\,{\rm sinc}}

\def\tbfg{\textbf{Fig.}}

\newcommand{\Real}{\mathbb R}

\newcommand{\bbP}{\mathbb P}
\newcommand{\bbQ}{\mathbb Q}
\newcommand{\bbQP}{{\mathbb Q}\times{\mathbb P}}

\newcommand{\bbg}{\mathfrak g}

\newcommand{\bbG}{\mathfrak{G}} 
\newcommand{\bbf}{\distribution}

\newcommand{\cP}{\mathcal{P}}

\newcommand{\cT}{\mathcal{T}}
\newcommand{\cL}{\mathcal{L}}

\newcommand{\commentOUT}[1]{}

\newcommand{\commentSG}[1]{\textcolor{green}{#1}}   

\newtheorem{thm}{Theorem}
\newtheorem{lem}{Lemma}
\newtheorem{pro}{Proposition}
\newtheorem{cor}{Corollary}
\newtheorem{defn}{Definition}

\journal{Physica D}

\begin{document}

\begin{frontmatter}



\title{Hamiltonian Monte Carlo with Asymmetrical Momentum Distributions} 


\author{Soumyadip Ghosh, Yingdong Lu, Tomasz Nowicki} 

\affiliation{organization={IBM T.J. Watson Research Center},
            addressline={1101 Kitchawan Rd}, 
            city={Yorktown Heights},
            state={New York},
            postcode={10598}, 
            country={U.S.A}}

\begin{abstract}
Existing rigorous convergence guarantees for the Hamiltonian Monte Carlo (HMC) algorithm use Gaussian  auxiliary momentum variables, which are crucially symmetrically distributed. 
We present a novel convergence analysis for HMC utilizing new dynamical and probabilistic arguments. The convergence is rigorously established under significantly weaker conditions, which among others allow for general auxiliary distributions. 
In our framework, we show that plain HMC with asymmetrical momentum distributions breaks a key self-adjointness requirement. We propose a modified version of HMC, that we call the Alternating Direction HMC (AD-HMC), which overcomes this difficulty. Sufficient conditions are established under which AD-HMC exhibits geometric convergence in Wasserstein distance. The geometric convergence analysis is extended to when the Hamiltonian motion is approximated by the leapfrog symplectic integrator, where an additional Metropolis-Hastings rejection step is required. Numerical experiments suggest that AD-HMC can generalize a popular dynamic auxiliary scheme to show improved performance over HMC with Gaussian auxiliaries.
\end{abstract}




\end{frontmatter}



\section{Introduction}

\commentOUT{I suggest more emphasis on the W result, the the analytical one, then algorithmic one, delegating the algorithm inserts to the end of the Intro }

Hamiltonian Monte Carlo (HMC) belongs to the wider class of Markov Chain Monte Carlo (MCMC) algorithms~\cite{hastings70,gelfand90} that approximate the difficult-to-compute density of a target probability measure by running a Markov chain whose invariant measure coincides with the target distribution. Let  $\distribution(q)$ denote the target density of interest and support of $\distribution(q)$ is on a metric space $(\bbQ,\distance)$ 
with $\distance$ being a metric in $\bbQ$. 

The density $\distribution(q)$ is of form $\distribution(q) = {\hat \distribution}(q)/C$, where ${\hat \distribution}(q)$ is easily queried but the normalizing constant $C= \int_\bbQ {\hat \distribution}(q) dq$ is hard to calculate. A core problem in modern statistics is in computing expectations with respect to such a density $\distribution(q)$, and forms a fundamental operation in many applications in machine/deep learning. For instance, in frequentist statistics, it relates to estimating coverage of a statistic using likelihood over data space. In 
Bayesian statistics, it can be generically stated as inferring the posterior distribution ${\mathfrak p}(q | x)$ of (unobservable) latent variables $q$ given observations $x$ from a user-modeled joint density ${\mathfrak p}(q, x)$. 
The inference target is ${\mathfrak p}(q | x) = {\mathfrak p}(q,x) / {\mathfrak p}(x)$, where the exact
marginal ${\mathfrak p}(x) = \int {\mathfrak p}(q,x) dq$ is usually hard to compute directly.  
Sample generation from density  $\distribution(q)$ is of independent interest in the emerging area of nonconvex optimization that seeks to efficiently find global optimal solutions by quickly exploring the decision space using HMC to discover all local optima \cite{chau22sghmc,gao21sghmc}; see the numerical experiments in Section~\ref{sec:expts}.

This motivates the wide use of MCMC in applications such as statistical inference~\citep{robert04}, inverse problems~\citep{stuart10}, artificial intelligence~\citep{andrieu03}, molecular dynamics~\citep{lilievre10} and global optimization. The early success of the MCMC approach can be attributed to simple, elegant and provably convergent algorithms such as the Metropolis-Hastings method~\citep{hastings70}, in which a user chosen auxiliary distribution and an acceptance/rejection mechanism helps pick and accept candidate samples as being from $\distribution(q)$. But MCMC algorithms suffer from slow convergence in high dimensional statistical computing in fields like molecular dynamics and artificial intelligence.  

Sample generation in the standard HMC (Algorithm~\ref{algo:hmc}) 
is driven by a dynamical system. Write the energy of the target distribution as $U(q) = -\log \distribution(q) = -\log {\hat\distribution(q)} + \log C$; we assume that an oracle can efficiently calculate the gradient of $U(q)$. In each iteration, the HMC Algorithm spreads (or lifts) the current iterate $q\rightarrow (q,p)$, where $p$ is an auxiliary set of {\em momentum} variables in space $\bbP$, which 
is also ${\mathbb R}^m$. The momentum $p$ is sampled from a probability measure with density function $\bbg(p) = e^{-V(p)}$, referred to as the \emph{momentum} or \emph{auxiliary} distribution with $V(p)$ its  {\em kinetic} energy. A (deterministic) transformation ${\cR}:(q,p)\mapsto (Q(T),P(T))$ is then applied using the following system of differential equations for a pre-selected time~$T>0$:
 \begin{align}
 \label{eqn:hamiltonian}
 {\dot Q}(t) =\frac{\partial H}{\partial p},\quad {\dot P}(t) =-\frac{\partial H}{\partial q},\quad\;&
 (Q(0), P(0))=(q,p)\,,
 \end{align}
where the total energy, or Hamiltonian, $H(q, p)=U(q)+V(p)$ of the system for time $t\in[0,T]$ is preserved by the motion generated by such dynamics. The $P$ variables are then dropped by applying the projection operator $\measureRho_\bbQ : (Q(T),P(T))\mapsto Q(T)$ to obtain next iterate. 
The terminology of $q$ and $p$ being the position and momentum variables, and $H(q, p)$ corresponding to the  energy arose in physics applications, where HMC first saw  wide-spread use \cite{DuaneEtAl87}. 
Since the parameter $T$ will remain fixed in this paper, we will omit it when there is no confusion.

The power of HMC stems from the energy-preserving property when obtaining the candidate solutions even for large $T$, since this also preserves the joint density of $(q,p)$ and hence (if the motion is exactly implemented) no new samples need  to be rejected. This leads to its relatively higher effectiveness compared to classical MCMC even in the high dimensional settings used in deep learning with neural networks~\cite{neal93, gelman14, jasche10, betancourt17}. With the rise of high-performance software implementations such as Stan~\cite{carpenter17, stan17}, the method has now become a pervasive tool across many scientific, medical, and industrial applications. Successful applications of HMC in scientific computing and discovery can be found in seismic research ~\cite{https://doi.org/10.1029/2019JB018428},~\cite{DELIMA2023128618}, and~\cite{10.1093/gji/ggae112}, as well as in physics~\cite{10.1093/gji/ggy496} and~\cite{PhysRevResearch.6.033142}. 

A rich literature has emerged on the important theoretical problem of when and how fast the HMC algorithm converges, but with a crucial limitation that the auxiliary distribution $\bbg(p)$ be chosen as a (conditional) Gaussian. Of particular interest has been the conditions under which geometric convergence can be ensured for this case, which has been achieved through different approaches such as comparison theorems for differential equations~\citep{chenvempala2019}, 
Harris recurrence techniques~\citep{bou-rabee2017} and coupling~\citep{bou-rabee2020}.  Some of the various conditions identified for geometric ergodicity of HMC are not easy to verify (e.g.~\citep{bou-rabee2017}), some~\citep{Durmus2017OnTC, rapidmixing, livingstone2019} cannot lead to explicit expression of the convergence rate on the HMC parameters, while others~\citep{bou-rabee2020} heavily depends on delicate tricks for Gaussian $V(p)$ distributions, and require log-concave properties of the target density function. Meanwhile, in many applications in natural language processing~\cite{NLPMaria}, biology~\cite{https://doi.org/10.1002/gepi.1370070605} and physics~\cite{OSTMEYER2021107978}, probability distributions that are supported on discrete or asymmetric domains with possibly asymmetric target distribution can often be observed, it is thus desirable to extend HMC methodology and analysis to cope with these difficulties.

\textbf{Contributions.}
The main thrust of our paper is in the presentation of a new set of analytic techniques for HMC style algorithms. 
In contrast to the most previous work, 
we utilize new developments in geometric understanding of Markov chains, in conjunction with techniques in dynamical systems and probability to obtain qualitative results on both convergence and convergence rate. More specifically, by examining the coarse Ricci curvature defined though the measures related to the HMC Markov chain, coupled with detailed calculations for the Hamiltonian dynamical systems, we are able to establish geometric convergence in Wasserstein metric for a broader class of HMC algorithm settings than in the literature. In particular, we drop the key symmetry restriction on the auxiliary distribution $\bbg(p)$ (that is $\bbg(-p) = \bbg(p)$; see Sec.~\ref{sec:prelim}) thus generalizing the HMC algorithm to enable the use of a much broader class of any asymmetric auxiliary distribution.
Theorem~\ref{thm:two}, further relaxes the global condition of strong log-concavity on the target distribution used in previous results to only requiring that strong  log-concavity hold outside a ``small'' set (defined in~\eqref{defn:pseudo-small}). 
While we do not necessarily provide the tightest convergence rate estimate in comparison to existing results, our results focus on convergence in Wasserstein metric, which is not only a weaker notion hence covering more applications in practice but can also provide more robust bound for high dimensional problems as discovered in recent literature; see e.g.~\cite{HairerEtAl11} and~\cite{Durmus2015}.


%
%

The use of general asymmetric $\bbg(p)$ necessitates a modification of the standard HMC Algorithm~\ref{algo:hmc}, which we present in Algorithm~\ref{algo:adhmc}. The main modification requires that the procedure alternates Hamiltonian motion in forward and backward directions for the same length $T$, and hence we call the algorithm the Alternating Direction HMC (AD-HMC). Our functional analysis approach allowed us in~\cite{GHOSH2022107811} (see Sec.~\ref{sec:notion of Conv} for a brief overview) to identify key properties of Hamiltonian motions that are sufficient to first establish convergence of the AD-HMC iterates. The self-adjointness of the HMC operator is a key characteristic that enables our proof of convergence, and the standard HMC operator with an asymmetric momentum distribution is not self-adjoint. The modified Alternating Direction HMC (AD-HMC) Algorithm~\ref{algo:adhmc} rectifies this by applying the HMC operator and its adjoint in alternating steps, and amounts to taking Hamiltonian motion in forward and backward directions. 
\commentOUT{\emph{take to Sec~\ref{sec:convL2}}: This significantly expands the class of HMC algorithms for which convergence is rigorously established, among others by dropping the symmetry restriction on momentum distributions. The proofs provide a simple and intuitive understanding of the working of HMC and illustrate \emph{why} its iterates converge. Moreover, our observations on the functional and probabilistic structures of the algorithms lead to a significantly shortened presentation compared to previous work in the literature. 
}

The Hamiltonian dynamics in HMC algorithms (including AD-HMC) are typically implemented by discrete numerical approximations such as the leapfrog integrator~\cite{leimkuhler04}. Since these approximations do not preserve the Hamiltonian $H$ exactly, an additional Metropolis-Hastings rejection step is imposed in AD-HMC as a correction. Lemma~\ref{lem:adhmc_mh_reversibility} shows that the MH corrected AD-HMC is time-reversible and hence converges. Theorem~\ref{thm:three} confirms that such practical discretization schemes along with the correction step also converge geometrically broadly under conditions similar to those imposed in Theorem~\ref{thm:two} for exact AD-HMC. 

A popular technique to speed up standard HMC methods is to dynamically update the covariance matrix of a Gaussian auxiliary $\bbg(p)$ with the goal~\citep{betancourt17} being to match the contours of $V(p)$ to $U(q)$.
Motivated by the form of the AD-HMC motion evolution derived in Proposition~\ref{prop:d QP d qp}, in Section~\ref{sec:algo} we generalize this approach to propose an  
adaptive scheme that dynamically constructs $\bbg(p)$ as a general mixture of Gaussians that aims to approximate the target $\distribution(q)$ up to all of its the modes. 
Section~\ref{sec:expts} describes our initial numerical experiments that suggest the potential benefit of considering asymmetrical momentum distributions. 
A three-dimensional simulated global optimization example is studied where the optimization objective has multiple local optima with widely varying regions of attraction. 
We demonstrate that the AD-HMC based heuristic significantly outperforms standard HMC with adaptively learned Gaussian distributions and 
motivates the use of general auxiliary distributions $\bbg(p)$ with the adaptive AD-HMC  heuristic, 
displaying the great potential of AD-HMC. 

The rest of the paper is organized as follows, basic information on HMC algorithms and technical preparations are prepared in Sec.~\ref{sec:prelim}; 
the core of our convergence analysis is presented in Sec.~\ref{sec:geometric_convergence}; then, AD-HMC is introduced and analyzed in Sec.~\ref{sec:perturbations}; 
the adaptive AD-HMC algorithm is presented in Sec.~\ref{sec:algo} and Sec.~\ref{sec:expts} provides results of experiments with the new algorithm; 
and 
the paper is concluded with a summary of our findings in Sec.~\ref{sec:conclusions}.

\section{Preliminaries}\label{sec:prelim}

%
%
The proof of the convergence of HMC algorithm rely on the following conditions abstracted from the Hamiltonian motion defined by~\eqref{eqn:hamiltonian}. 
\begin{itemize}
\item
The \emph{target distribution} is proportional to a function $0\le \bbf:\bbQ\to\Real$, integrable with respect to a \emph{reference measure} $dq$, $\int_\bbQ \bbf(q)\,dq<\infty$
\item
The \emph{auxiliary distribution} is a function $0\le \bbg:\bbP\to \Real$, $\int_\bbP \bbg(p)\,dp=1$, where $dp$ is a \emph{reference measure} on $\bbP$.
\end{itemize}
The measurable \emph{invertible} motion $\cR:\bbQP\to\bbQP$,
$\cR(q,p)=(Q,P)$, related to the target and the auxiliary distributions
has the following \emph{invariance properties}: 
\begin{itemize}
    \item[A1.] $\bbf(Q)\cdot\bbg(P)=\bbf(q)\cdot \bbg(p)$ (conservation of the Hamiltonian energy), and
    \item[A2.] for any integrable function $A:\bbQ\times\bbP\to \Real$ we have $\iint_{\bbQ\times\bbP} A\circ {\cR}\,\, d(q,p)=\iint_{\bbQ\times\bbP} A\,\, d(q,p)$ (conservation of Lebesgue measure by the Hamiltonian motion).
\end{itemize} 
Additionally, we assume that the motion is irreducible (ergodic), which means there are no nontrivial measurable invariant sets. For simplicity we assume a \emph{coverage property}: 
\begin{itemize}
    \item[A3.] $ \measureRho_\bbQ \,(\, {\cR}(q,\bbP)\,)\,=\,\bbQ$ for (almost) every $q\in \bbQ$, with $\measureRho_\bbQ (A) := \{q\in \bbQ: \exists p\in \bbP, (q,p)\in A\}$ for any $A\in \bbQP$.
\end{itemize} 
This compactly represents (using a slight abuse of notation ${\cR}(\cdot,\cdot)$ and $\measureRho_\bbQ(\cdot)$) the property that every point $Q\in \bbQ$ can be reached from (almost) any $q\in\bbQ$ with a lift by an appropriate $p$ for the motion ${\cR}$ in one step. 

The HMC iterations are defined below.
We assume that the initial state is sampled from an arbitrary distribution that is absolutely continuous with respect to the \emph{reference measure} 
$dq$ on $\bbQ$, in case of the Hamiltonian motion it will be the Lebesgue measure on a Euclidean space.

\subsection{HMC as a Dynamical System}
\subsubsection{Function space} From an analytical point of view, the density function of a probability measure, well-defined when the measure is absolutely continuous with respect to the reference (Lebesgue) measure, can also be viewed as a member in a proper functional space. 
First, let the target measure on the space $\bbQ$ (which we can assume to be its support) be expressed as a density $\bbf$ with respect to the reference measure $dq$ on $\bbQ$. 
The abstract (or ideal) HMC uses an auxiliary measure with a density  $\bbg$ on the space $\bbP$ with respect to the reference measure $dp$ there. A step of HMC is a realization of an operator $\cT$ acting  
on the densities $h:\bbQ\to\Real$ belonging to 
$$L^2_\bbf:=\left\{h:||h||_\bbf^2=\int h^2/\bbf<\infty\right\},$$
where the integration is with respect to the reference measure $dq$. 
To simplify notation we shall skip the subscript $\bbf$ and write $L^2$ and $||h||^2$.

\subsubsection{Operator} The operator $\cT$ 
first constructs the joint distribution $h(q)\bbg(p)$ on $\bbQ\times\bbP$ (in the lifting step), then performs the Hamiltonian motion in the product space $(q,p)\mapsto (Q,P)={\cR}(q,p)$ producing a joint distribution  $h(Q)\cdot \bbg(P)=(h\cdot\bbg)\circ {\cR}(q,p)$ and finally projects this transported density
along the direction of $\bbP$ on its marginal on space $\bbQ$. 
More specifically,
\begin{equation}
  \label{eqndef:cT}
  \cT h(q)=\int_\bbP (h\cdot\bbg)\circ {\cR}(q,p)\,dp=\int_\bbP h(Q)\bbg(P)\,dp\,.
\end{equation}

The progression of HMC algorithm applied to an initial density $h$ can be expressed as the sequence of iterations of the operator $\cT^n h $, where $\cT^{n+1}=\cT\circ\cT^n$.

\commentOUT{
The evolution of iterates $q$ is naturally modeled as a Markov chain. Next step depends only on the previous one and an independent and randomly generated momentum vector $p$. Hence, we can equivalently study the evolution of the density functions of the Markov chain. Our HMC analysis combines a dynamical system on a functional space with methods in probability theory.}

\subsection{HMC as a Markov chain}

Modeling the evolution of Algorithm~\ref{algo:hmc} as a Markov chain defined on the space $\bbQ$, its transition probability can be defined as follows. Given the fixed parameter $T$ and initial position $q\in \bbQ$, 
the map 
\[
\Pi_{q}:=\measureRho_\bbQ \circ {\cR}(q,\cdot): \bbP \rightarrow \bbQ\,,
\]
with $\measureRho_\bbQ$ being the usual projection from $\bbQ\times \bbP$ to $\bbQ$, is an onto map. Hence, for any given probability measure $\bbG$ on $\bbP$ (with density $\bbg(p) = e^{-V(p)}$),
the map $\Pi_{q}$ induces a push forward $(\Pi_{q })_\# (\bbG)$, a probability measure on $\bbQ$ defined as $(\Pi_{q })_\#  (\bbG)(A) = \bbG(\Pi_{q }^{-1}(A))$ for any measurable set $A$, with $\Pi_{q }^{-1}(A)$ denoting its pre-image under $\Pi_{q }$. 
This push forward provides the transition probability for our Markov chain, $\pr(q, A) := (\Pi_{q})_\#  (\bbG)(A)$.

\subsection{Functional Analysis Approach to Convergence}\label{sec:notion of Conv}

With the Markov chain and dynamical system defined above, various notions of convergence need to be discussed for understanding the main results of the paper. 

From the analytical viewpoint, the convergence of the HMC is studied in terms of convergence of density functions in a proper functional space. In particular, the  $L^2$ space defined previously with respect to the reciprocal  of the density function $\bbf$ has been found to be very useful in establishing the convergence for a variety of iterative operators, see e.g. \cite{MV2000}. 
%
The space $L^2$ has a natural scalar product $\langle a,b \rangle=\int ab/\bbf$. Hence, the (strong) convergence in $L^2$ for a sequence of functions $h_n$ to a function $h$ is defined by
$\lim_{n\rightarrow \infty} || h_n-h|| =0$, and the 
weak convergence is defined as 
$\lim_{n\rightarrow \infty} \langle h_n-h, a\rangle =0$,
for any $a\in L^2$. 
In~\cite{CHEN2000281}, it is shown that $L^2$ geometric convergence is equivalent to geometric ergodicity, which is a well studied property in the Markov chain literature. 
In the space $L^2$, the \emph{adjoint} operator $\cT^\dag$ to $\cT$  is characterized by $\langle \cT a, b\rangle=\langle a,\cT^\dag b\rangle$ for any $a,b \in L^2$. A~\emph{self-adjoint} operator satisfies $\cT^\dag=\cT$.
By the invariance properties (A1) and (A2) described above
$\cT^\dag h=\int_\bbP (h\cdot\bbg)\circ {\cR}^{-1}$, and a sufficient condition for self-adjointness is that the auxiliary distribution is symmetric $\bbg(p)=\bbg(-p)$.  If $\cT$ is not self-adjoint (then $\bbg(p)$ must be asymmetrical) define ${\cT_A}=\cT^\dag\circ\cT$ (since $\langle \cT_A a, b\rangle=\langle \cT^\dag\circ\cT a, b\rangle= \langle \cT a, \cT b\rangle =\langle  a, \cT^\dag\circ\cT b\rangle= \langle  a, \cT_A b\rangle$).  Self-adjointness is crucial to the convergence result in the following Theorem~\ref{thm:strong conv}, whose detailed proof can be found in~\cite{GHOSH2022107811}.






\begin{thm}\label{thm:strong conv} 
  For any $h\in L^2$ the sequence of alternating iterations ${\cT_A}^n h$  converges strongly to the fix point $\alpha\bbf$, where $\alpha$ is a constant, i.e. $\lim_{n\ra \infty} \|{\cT_A}^n h-\alpha\bbf\|^2=0$. 
  If additionally $\cT$ is self-adjoint itself then $\cT^n h$ converges strongly to $\alpha\bbf$.
\end{thm}

We observe that $\cT h =\bbf\int_\bbP (h/\bbf)\circ\cR \cdot \bbg$ is in fact an averaging map
(\emph{Lemma~3.2 of \cite{GHOSH2022107811}})
thus by the convexity of $x\mapsto x^2$
the norm decreases under $\cT$: $\|\cT h\|<\|h\|$, 
(\emph{ibidem}), 
by the coverage assumption sharply, unless $h=\alpha\bbf$. In the space $L^2$, bounded sequences have weak accumulation points. Using self-adjointness it is then proven 
(\emph{Corollary~5.3 of \cite{GHOSH2022107811}})
that each accumulation point of the sequence of iterations has the same norm and thus 
(\emph{ibidem}), 
must be of form $\alpha\bbf$, where the value $\alpha$ is deduced from integral invariance
$\int_\bbQ h =\int_\bbQ \cT h$
(\emph{Lemma~3.2 of \cite{GHOSH2022107811}} again).
Hence the whole sequence converges. Meanwhile the convergence of the norms to the norm of the limit provides 
the proof of strong convergence (\emph{Proposition~5.4 of \cite{GHOSH2022107811}}).

This significantly expands the class of HMC algorithms for which convergence is rigorously established, among others by dropping the symmetry restriction on momentum distributions. 

\medskip

The proofs provide a simple and intuitive understanding of the working of HMC and illustrate \emph{why} its iterates converge. Moreover, our observations on the functional and probabilistic structures of the algorithms lead to a significantly shortened presentation compared to previous work in the literature. 

\subsection{Probabilistic Approach to Convergence}\label{sec:notion of Conv prob}
\begin{defn}
\label{defn: wasserstein_defn}
For any two measures $\mu$ and $\nu$ on metric space $(\bbQ,\distance)$, the Wasserstein distance $W_p(\mu, \nu)$ for $p> 0$ is defined as,
\begin{align}
\label{eqn:wasserstein_defn}
W_p(\mu, \nu) = \inf_{\gamma \in \Gamma(\mu. \nu)} \left[\int_{{\mathbb Q} \times {\mathbb Q}} \distance^p(x,y) \gamma(dx, dy)\right]^{\frac{1}{p}},
\end{align}
where $\Gamma(\mu, \nu)$ denotes the set of measures on ${\mathbb Q} \times {\mathbb Q}$ that project on $\mu$ and $\nu$. 
\end{defn}

A closely related concept is that of the coarse Ricci curvature of a Markov operator, developed in \cite{OLLIVIER2009810} and ~\cite{joulin2010}. Ricci curvature of a Riemannian manifold is interpreted as a measure of deformation along parallel transformations. 
\begin{defn}
On the space ${\mathbb Q}$, with the Markov operator $\calp$, for each pair $(q_1, q_2)\in {\mathbb Q}\times {\mathbb Q}$, the coarse Ricci curvature $\kappa(q_1, q_2)$ in the direction of $(q_1, q_2)$, is defined as,
\begin{align*}
\kappa(q_1, q_2):=1-\frac{W_1(\mu(q_1), \mu(q_2))}{\distance(q_1, q_2)},
\end{align*}
where $\mu(q)$ denotes the probability distribution of the Markov chain governed by $\calp$ with initial state $q$.
\end{defn}
Denote $\kappa:= \inf_{(q_1, q_2) \in \bbQ\times \bbQ}\kappa(q_1, q_2)$. It is demonstrated in~\cite{OLLIVIER2009810} (Corollary 21) and~\cite{joulin2010} that the assumption $\kappa >0$ ensures the existence of a unique invariant measure $\pi$ for the Markov chain, moreover, for any measure $\mu$ on $\bbQ$, the following geometric convergence holds,
\begin{align}
\label{eqn:W_1GeoConv}
W_1(\calp^N \mu, \pi)\le (1-\kappa)^N W_1(\mu, \pi),
\end{align}
with $\calp^N \mu$ denotes the measure of the Markov chain after $N$ steps starting from $\mu$.

\section{Geometric Convergence 
in 
$W_1$}

\label{sec:geometric_convergence}
In this section, we discuss the geometric convergence of the HMC iterations under the Wasserstein metric $W_1$. This requires a careful analysis of the end-point $(Q,P)$ of the Hamiltonian motion undertaken in each iteration starting from $(q,p)$. We will on occasion expand to the notation $(Q(q,p), P(q,p))$ to stress the dependence on the initial condition. The results in this section assume that the spaces $\bbQ, \bbP$ are finite dimensional real Euclidean spaces $\Real^d$ in order to have straightforward definitions of the partial derivatives of $Q$ and $P$ with respect to $(q,p)$.

\subsection{
Global Condition on the Hessian}
\label{sec:original}
%
Recall that by the Ricci curvature arguments, the geometric convergence in the Wasserstein distance is established if the quantity $W_1(\mu(q_1), \mu(q_2))$ can be shown to be contracting comparing to the initial positions of the Hamiltonian motion. 
The precise calculation of $W_1(\mu(q_1), \mu(q_2))$ requires solving the well-known Monge-Amp\'ere equation, a nonlinear elliptic equation whose solutions are hard to obtain, see e.g. \cite{villani2008optimal}. 
Here, we identify the conditions under which an upper bound can be derived to this important quantity, and this then leads to a lower bound to curvature $\kappa$, which is still a positive number, hence the geometric convergence of the Markov chain to its stationary distribution. While there exist several similar results for the geometric convergence of HMC, we believe that our proof provided here is among the simplest, and it provides important insights on the related dynamics that of independent interest. 

\begin{thm}
\label{thm:one}
If $ \Big| \frac{\partial Q}{\partial q} \Big|\le \beta$, for some $\beta \in (0,1)$, where $|\cdot|$ denotes the operator norm for matrices,   then the Markov recursion converges geometrically in $W_1$ with a rate being at least $(1-\beta)$. 
\end{thm}

\begin{proof} 
Recall that to bound $\kappa$ away from zero, we will need to provide a uniform ( and $\le 1$) upper bound to the quantity $\frac{W_1(\mu(q_1), \mu(q_2))}{\distance(q_1, q_2)}$ for any pair of $q_1$ and $q_2$. This upper bound will be achieved by two relaxations.  First, we identify one member $\gamma_0$ in the set of joint distributions $\Gamma(\mu(q_1), \mu(q_2))$. By definition \ref{defn: wasserstein_defn} , the integration, $\int_{{\mathbb Q} \times {\mathbb Q}} \distance(x,y) \gamma_0(dx, dy)$, naturally provides an upper bound to $W_1(\mu(q_1), \mu(q_2))$. The second relaxation is on the calculation of the function $\distance(x,y)$ within the integration $\int_{{\mathbb Q} \times {\mathbb Q}} \distance(x,y) \gamma_0(dx, dy)$.

The basic idea of the first relaxation is the same as that of the common random number generator in simulation literature. M. Talagrand \cite{Talagrand1996}, who attributed the idea to M. Frechet, used it to prove a version of the logarithmic Sobolev inequalities, which is closely related to the geometric convergence of Markov chains. 

Suppose that $\bbG$ is the selected auxiliary distribution on ${\mathbb P}$ with density $\bbg$. Define a map $\cL: {\mathbb P} \rightarrow {\mathbb Q} 
\times  {\mathbb Q}$ with $ p \mapsto (Q(q_1, p), Q(q_2, p))$, for any given $(q_1, q_2)\in {\mathbb Q} \times  {\mathbb Q}$.  The map $\cL$ thus induces a measure $\gamma_\bbG$ on ${\mathbb Q} \times  {\mathbb Q}$  that  can be viewed as $\cL_\sharp \bbG$. More specifically,  for any subset $A, B \subset {\mathbb Q}$, the measure $\gamma_\bbG$ has the following representation,
\begin{align*}
\gamma_\bbG (A\times B) = \int_{{\tilde{ A}\,\cap\, \tilde{ B}}} \bbg(p) dp,
\end{align*}
with $\tilde{A}: = \{ p\in {\mathbb P} \,|\, Q(q_1, p) \in A\}$  and  $\tilde{B}:=\{ p \in {\mathbb P} \,|\,  Q(q_2, p) \in B\}$. With this definition, $\gamma_\bbG$ is easily verified to satisfy $\gamma_\bbG\in\Gamma(\mu(q_1), \mu(q_2))$.
Therefore,
\begin{align*}
W_1(\mu(q_1), \mu(q_2))  \le \int \int \distance(x, y) \gamma_\bbG (dx, dy). 
\end{align*} 
This approach is naturally seen to couple the two motions with the same momentum generated from the common distribution $\bbG$ on ${\mathbb P}$.

From the property of integrability for push-forward measure, see, e.g. \cite{bogachev2007measure}, we know that, 
\begin{align*}
 \int \int \distance(x, y) \gamma_\bbG (dx, dy) =\int \distance (Q(q_1, p), Q(q_2,p)) \bbg(p) dp 
\end{align*}
For the second relaxation, the quantity $\distance (Q(q_1, p), Q(q_2,p)) $ in the above integration can be further upper bounded by the length of one specific  curve that connects $Q(q_1, p)$ and $Q(q_2,p)$. 
Consider the case when $\distance$ is given by a norm $\|\cdot\|$. For any $t \in [0,1] \rightarrow {\mathbb Q}$, let $\eta(t)=Q(t q_2+ (1-t) q_1, p)$. Thus, ${\dot \eta}(t) = \frac{\partial Q}{\partial q}\cdot (q_2-q_1)$.  Hence, 
\begin{align*}
\distance (Q(q_1, p), Q(q_2,p)) &= \|Q(q_1, p)- Q(q_2,p)\|\\ &\le \int_0^1 \sqrt{ |{\dot \eta}(t)|^2}  dt \le \Big| \frac{\partial Q}{\partial q} \Big|\cdot\|q_2-q_1\|=\Big| \frac{\partial Q}{\partial q} \Big|\cdot \distance(q_1, q_2).
\end{align*}
\end{proof}
\commentOUT{Here again there is no link between the distance $\distance(Q_1,Q_2)$, the distance $|q_1-q_2|$ and the how the derivative is calculated. I think when we are floating on the abstract we need to be careful.}

%
%
%
The bound $ \Big| \frac{\partial Q}{\partial q} \Big|\le \beta <1$ 
is on the sensitivity of the Hamiltonian motion output $Q$ on the starting point $q$, and applies for any momentum $p$. We will show below that this bound can be obtained if additional conditions are imposed on the form of the densities $\distribution$ and $\bbg$. 

As before, let $\bbg(p)$ have density $\exp(-V(p))$, and with Hamiltonian form $H(Q,P) = U(Q) +V(P)$, denote 
 $V''=\partial^2 H/\partial P^2$ and $U''=\partial^2 H/\partial Q^2$. 
%
%
%
Define averages $\bar{V}=\bar{V}(t)=\frac{1}{t}\int_0^t V''(P(s))\,ds$ and $\bar{U}=\bar{U}(t)=\frac{1}{t}\int_0^t U''(Q(s))\,ds,$
where $P, Q$ are the solutions of~\eqref{eqn:hamiltonian}.

Lemma~\ref{lem:evol general} studies the general evolution of the four initial-configuration dependence terms including the $\partial Q / \partial q (t)$.  Proposition~\ref{prop:d QP d qp} below provides a representation of solution of the evolution equations.

\begin{lem}[Evolution of the dependence on the initial configuration]\label{lem:evol general}
  When the Hamiltonian is given by $H(Q,P) = U(Q) +V(P)$,
  the  derivative of the motion $(Q,P)$ with respect to the starting configuration $(q,p)$ satisfy the following time evolution equation:
 \begin{equation}\label{eqn: ODE general}
  \frac{\partial}{\partial t}
  \left(
    \begin{array}{cc}
     \frac{\partial Q}{\partial q} & \frac{\partial Q}{\partial p} \\
     \frac{\partial P}{\partial q} & \frac{\partial P}{\partial p}
    \end{array}
  \right)
  =
  \left(
  \begin{array}{cc}
    0& V'' \\
    -U''& 0
  \end{array}
  \right)
  \cdot
  \left(
    \begin{array}{cc}
     \frac{\partial Q}{\partial q} & \frac{\partial Q}{\partial p} \\
     \frac{\partial P}{\partial q} & \frac{\partial P}{\partial p}
    \end{array}
  \right)\,
\end{equation}
with initial condition,
\begin{equation}
  \left(
    \begin{array}{cc}
     \frac{\partial Q}{\partial q} & \frac{\partial Q}{\partial p} \\
     \frac{\partial P}{\partial q} & \frac{\partial P}{\partial p}
    \end{array}
  \right)_{t=0}=
 \left(
  \begin{array}{cc}
    I & 0 \\
    0 & I
  \end{array}
  \right)
 \end{equation}
\end{lem}
Proof of Lemma~\ref{lem:evol general} is standard and can be found in~\cite{GHOSH2022107811}.

\medskip

\noindent
The functions $\cos$ and $\sinc$ in Proposition~\ref{prop:d QP d qp} below are well defined on bounded operators by their power series:
$\cos(x)=\sum_{n=0}^\infty (-1)^n x^{2n}/{(2n)!}$ and  $\sinc(x)=x^{-1}\sin(x)=\sum_{n=0}^\infty (-1)^n x^{2n}/(2n+1)!$, which is well defined even when $x^{-1}$ is not.
{We remark that in both power series only even exponents are used.}

\begin{pro}
  \label{prop:d QP d qp}
  Assume that both the target and the auxiliary distributions are strongly log-concave over their domains. Denote 
  {by $A$ and $B$ the matrices where 
  $A^2:=\bar{V}\bar{U}$ and $B^2:=\bar{U}\bar{V}$ hold}.
  Then the solution of the evolution equation \eqref{eqn: ODE general} 
\begin{align*}
 &\left(
    \begin{array}{cc}
     \frac{\partial Q}{\partial q}& \frac{\partial Q}{\partial p } \\
     \frac{\partial P}{\partial q} & \frac{\partial P}{\partial p}
    \end{array}
  \right)(t)=
  \left(
   \begin{array}{cc}
    \cos \left(tA\right)& t\bar{V}\sinc\left(tB\right)\\
    - t\bar{U}\sinc\left(tA\right)& \cos \left(tB\right)
  \end{array}
  \right)\,.
\end{align*}
\end{pro}
By logarithmic concavity and by continuity for  each $Q$ and $P$ both $V''$ and $U''$ are positive symmetric operators, and so are their time averages $\bar{V}$ and $\bar{U}$.  The functions $\cos(x)$ and $\sinc(x)=x^{-1}\sin(x)$ used above are well defined by their power series, which are absolutely convergent for all bounded operators, the latter even when $x^{-1}$ is not well defined.
\begin{proof}
The expression follows from standard calculations on the exponential solutions of the linear evolution equations in Lemma \ref{lem:evol general}.
Specifically, 
\emph{strong logarithmic concavity} on the entire domain implies that the Hessian of $H(Q,P)$ 
can be bounded (for example in terms of their spectra) away from $0$ and $\infty$ uniformly. Then $\bar{V}$ and $\bar{U}$ are uniformly bounded, and so are $A$ and $B$.
\end{proof}
The standard Gaussian satisfies this strong log-concavity assumption.
From the solution, observe that with a well chosen positive but small $T$ such that  $0<T<\pi/|A|$, we have  $\Big| \frac{\partial Q}{\partial q} \Big| = |\cos(tA)|\le \beta <1$.
\begin{cor}\label{cor:geom_for_log_concave}
Suppose both the target $\distribution$ and auxiliary $\bbg$ are strongly log-concave, and the AD-HMC algorithm implements exact Hamiltonian dynamics for $T$ such that $0<T<\pi/|A|$ with $A^2={\bar{V}\bar{U}}$. Then the algorithm 
converges geometrically to the target $\distribution$. 
\end{cor}

\subsection{
Relaxed Condition on the Hessian}
\label{sec:extensions}

In this section, we will relax the uniform strongly logarithmic concave  conditions over the entire space. Instead, we only require that it holds outside a compact set. To facilitate the analysis, we assume the compact set to be, $B_R=\{q:\distance(q, 0) \le R\}$, in the metric space. 
\begin{thm}
\label{thm:two}
If $ \Big| \frac{\partial Q}{\partial q} \Big|\le \beta <1$ outside $B_R=\{q:\distance(q, 0) \le R\}$, and  $ \Big| \frac{\partial Q}{\partial q} \Big|$ is bounded by one within $B_R$, for some real number $R>0$, then there exists an $\iota>0$ such that the Markov recursion converges geometrically in Wasserstein metric with a rate  at least $\beta'= \max \{1-\iota, \frac{1+\beta}{2}\}$.
\end{thm}
\begin{proof}
This proof will again make use of the coarse Ricci curvature arguments. Recall that, it suffices to show that for any pair $(q_1, q_2)$, $W_1(\mu(q_1), \mu(q_2)) \le \beta' \distance(q_1, q_2)$. We need to discuss two different cases. First, for the case of $(q_1, q_2)\notin B_R\times B_R$, from Lemma \ref{lem:bivariate}, we have, $W_1(\mu(q_1), \mu(q_2)) \le \frac{1+\beta}{2}\distance(q_1, q_2)$. Second, for the case of  $(q_1, q_2)\in B_R\times B_R$, define a joint distribution ${\hat \gamma}$ as follows. Let $\xi$ be an independent uniform distribution in $[0,1]$. When $\xi\le \iota$, ${\hat \gamma}(\mu(q_1), \mu(q_2))=\gamma_\nu(Q_1\in dy, Q_2\in dz) = \nu(Q_1(q_1,y)|Q_1(q_1,y) = Q_2(q_2,z))$, i.e. $Q_1$ and $Q_2$ will be at the same position, which distributes according to the probability distribution $\nu$; when $\xi> \iota$, each one will be going to be 
${\hat \gamma}(\mu(q_1), \mu(q_2))=(1-\iota)^{-1}[\mu(q_1)(dy)-\iota\nu(Q_1\in dy)](1-\delta)^{-1}[\mu(q_2)(dz)-\iota \nu(Q_2\in dz)]$, with ${\tilde \gamma}(\mu(q_1), \mu(q_2))$ , similar to the approach in e.g~\cite{rosenthal2002} and ~\cite{roberts2004}. The nonnegativity of ${\hat \gamma}(\mu(q_1), \mu(q_2))$ is guaranteed by the smallness of $B_R$ demonstrated in Lemma~\ref{lem:smallset} when $\iota\le \delta$.  Furthermore, it is easy to verify that the marginal distributions will not change, hence, ${\hat \gamma} \in \Gamma(\mu(q_1), \mu(q_2))$. Meanwhile, we can see that $\int\int \distance(x,y) {\hat \gamma}(\mu(q_1), \mu(q_2))\le (1-\iota) \distance(q_1, q_2)$ because the  probability that $\distance (\mu(q_1), \mu(q_2))=0$ is at least $\iota$. This leads to $W_1(\mu(q_1), \mu(q_2)) \le (1-\iota)\distance(q_1, q_2)$. The desired estimation of the coarse Ricci curvature follows from these two cases. 
\end{proof}
Theorem~\ref{thm:two} thus relaxes the strong log-concavity on $\bbg$, and in particular covers any auxiliary (symmetric or not) distribution that has a corresponding non-empty $B_R$.
The proof of Theorem~\ref{thm:two} relies on the following definition and two lemmas.
\begin{defn}
\label{defn:pseudo-small}
A set $C$ in the state space is called \emph{small} (or $(n_0, \delta)$-\emph{small}) if there exists $n_0\in {\mathbb Z}$ and $\delta >0$,  and a probability measure $\nu$  on the state space  such that for any $x \in C$, 
$\quad \pr^{n_0} (x, \cdot) \ge \delta \nu(\cdot).$
\end{defn}
\begin{lem}
\label{lem:smallset}
$B_{R}$ is a small set.
\end{lem}
\begin{proof}
To show that $B_R$ is a small set, we need to construct a probability measure $\gamma$ on the state space, such that $\pr(x, A) \ge \epsilon \gamma(A)$ for any Borel set $A$. Suppose that for any $q\in B_R$, there exists a $\rho(q)>0$, such that there exist a measure $ \gamma_q$ and an $\epsilon_q>0$ and that $\pr(x, A) \ge \epsilon_q \gamma_q(A)$ for any $x\in B_{\rho(q)} (q)$. Then by the compactness of ${\bar B}_R$ it can be covered by a finite number of them, $B_{\rho(q_1)} ,B_{\rho(q_2)}, \ldots, B_{\rho(q_N)}$, and this will imply that $B_R$ is small. 

So now, we only need to construct $ \gamma_q$ locally. For each $q$, consider the density function $\pr(x, dy)$, for $x\in B_{\rho(q)} (q)$ for some small $\rho(q) >0$, and any $y$. From the push forward definition of the 
transition probability, we can see that, $\pr(x, dy)= {\mathfrak g}(y-x)dp$ where ${\mathfrak g}(\cdot)$ denotes the density function of the auxiliary distribution. Hence, by the fact that ${\mathfrak g}$ decay to zero, there exists an $x_0$ (might be outside $B_{\rho(q)} (q)$), such that ${\mathfrak g}(y-x)\ge \epsilon'_q {\mathfrak g} (y-x_0)$ with some $\epsilon'_q \in (0,1)$. Then the probability measure can be constructed to be proportional to   ${\mathfrak g} (y-x_0)$, then we can conclude that $\pr(x, A) \ge \epsilon_q \gamma_q(A)$ for any $x\in B_{\rho(q)} (q)$ with a proper $\epsilon_q>0$. 
\end{proof}

\begin{lem}
\label{lem:bivariate}
Under the conditions of Theorem \ref{thm:two}, for $C=B_{2R}$, and for all $(q_1, q_2) \notin C\times C$, 
\begin{align}
\label{eqn:bivariate_drift}
W_1(\mu(q_1), \mu(q_2)) \le \frac{1+\beta}{2}\distance(q_1, q_2).
\end{align}
\end{lem}
\begin{proof}
For the bivariate drift condition, we see that it is easy to verify for $(q_1, q_2) \in C^c\times C^c$. More precisely, since both $q_1$ and $q_2$ are outside $B_{2R}$, there must exist a connecting path $\eta(t)$ which has at least two third of its length lying in the domain where $ \Big| \frac{\partial Q}{\partial q} \Big|\le \beta$, which is outside $B_{R}$, therefore, 
\begin{align*}
\distance (Q(q_1, p), Q(q_2,p)) 
&\le  \int_0^1 \sqrt{ |{\dot \eta}(t)|^2}  dt=\int_0^1 \sqrt{ \Big|\frac{\partial Q(t)}{\partial q} (q_2-q_1)\Big|^2}dt\\ 
& = \left[\left(\int_{[0,1] \cap {s: Q(t) \in C} } +\int_{[0,1] \cap {s: Q(t) \in C^c} }\right)\sqrt{ \Big|\frac{\partial Q(t)}{\partial q}\Big|^2}dt\right] \distance(q_1, q_2)\\ 
& \le 
 \left[\int_{[0,1] \cap {s: Q(t) \in C} } dt + \int_{[0,1] \cap {s: Q(t) \in C^c} }\beta dt\right]\distance(q_1, q_2) \\ 
 & \le 
 \left[1-\frac{2}{3} (1-\beta)\right] \distance(q_1, q_2),
\end{align*}
Next, let us consider the case of $(q_1, q_2) \in C \times C^c$, or equivalently $(q_1, q_2) \in C^c \times C$. Again, we have, there must exist a connecting path which has at least half of its length lying in the domain where $ \Big| \frac{\partial Q}{\partial q} \Big|\le \beta$, which is outside $B_{R}$, therefore,
\begin{align*}
\distance (Q(q_1, p), Q(q_2,p)) ]
\le & \left[\int_{[0,1] \cup {s: Q(t) \in C} } dt + \int_{[0,1] \cup {s: Q(t) \in C^c} }\beta dt\right]\distance(q_1, q_2)
\end{align*}
Therefore, we have,
\begin{align*}
\distance (Q(q_1, p), Q(q_2,p)) ]
\le \frac{1+\beta}{2}\distance(q_1, q_2),
\end{align*}
and we get the required bivariate drift bound. 
\end{proof}
Similar to Corollary~\ref{cor:geom_for_log_concave}, 
we have the following geometric convergence result.
\begin{cor}
Suppose both the target $\distribution$ and auxiliary $\bbg$ are strongly log-concave outside $B_R$ for a positive constant $R>0$, and the AD-HMC algorithm implements exact Hamiltonian dynamics for $T$ such that $0<T<\pi/|A|_{B_R^c}$ with $A=\sqrt{\bar{V}\bar{U}}$. Then the algorithm 
converges geometrically to the target $\distribution$. 
\end{cor}
%

Lemmata \ref{lem:smallset} and \ref {lem:bivariate}  are closely related to the important concepts of (pseudo-)small set and Lyapunov function of a Markov chain. In fact, under the same conditions, \eqref{eqn:bivariate_drift} can be viewed as a bivariate drift condition for $\al:=(1+\beta)/2$ for some distance function, then Lyapunov type of arguments, see, e.g.~\cite{rosenthal2002}, imply  geometric convergence in total variational distance via a coupling method that inspired some of our arguments here, see, e.g.  ~\cite{rosenthal2002}.
For more detailed information on (pseudo-)small set and Lyapunov function techniques, please see systematic treatments in ~\cite{doi:10.1081/STM-100002060},~\cite{roberts2004} and ~\cite{meyn2009markov}. 

\section{AD-HMC with Approximate Hamiltonian Motion}
\label{sec:perturbations}


In practice, numerical approximators are needed to implement the Hamiltonian motion~\eqref{eqn:hamiltonian} embedded within any HMC variants. Various numerical methods are available for this purpose, for example, the symplectic leapfrog integrator~\citep{verlet67} can be commonly seen in many literature and software implementations, and versions of the Runge-Kutta method can achieve high order of approximation. For a detailed discussion on these methods and their analysis, see~\cite{hairer2013geometric}. 
Such numerical integration algorithms cannot ensure that the Hamiltonian is exactly preserved, and so the Hamiltonian motion step is coupled with a Metropolis-Hasting-style acceptance/rejection step to ensure the convergence to the target distribution. 
This section starts with a description of why and how this rejection step ensures convergence in general Markov chains, and then applies it to the standard HMC (with symmetric auxiliaries) implemented with inexact Hamiltonian motion. We will then propose an appropriate rejection step for inexact Hamiltonian motion in AD-HMC and establish that the augmented algorithm will converge to the desired $\distribution$. Note that both HMC and AD-HMC, when the motion implementation is exact, trivially pass their respective rejection steps. Finally, we take a deeper look at the specific case of the symplectic leapfrog integrator~\citep{verlet67} and show that the geometric convergence results of Section~\ref{sec:extensions} extend to this inexact motion implementation. 

\subsection{Metropolis-Hastings Rejection}
\label{sec:reversibility}

Metropolis-Hastings (MH) rejection steps are designed to enforce time reversibility of the transitions of a Markov chain since this property implies the existence and uniqueness of its stationary distribution. 
Recall that for Markov chains defined on a general state space $X$
with transition probability kernel $P(\cdot, dy)$ is reversible (See e.g.~\cite{meyn2009markov})  with respect to a probability measure $\pi(\cdot)$ on $X$ if
\begin{align*}
\pi(dx)P(x,dy) = \pi(dy)P(y,dx), \quad \forall x,y \in X.
\end{align*}
If the Radon-Nikodym derivative of the stationary distribution $\pi$ is known up to a constant, say $\pi(dx)= Cf(x)dx$, and the transition kernel is expressed in density form as $P(x, dy) = p(y | x)dy$, then the reversibility is equivalently expressed as
\begin{align}\label{eqn:reversibility}
f(x)\,p(y|x)=&f(y)\,p(x|y), \quad \forall x,y \in X.
\end{align}

%
%
Suppose a Markov chain
starting at $x_t$ generates a proposal $x'$ for the next step with a probability $g(x'|x_t)$, playing the role of $y$ in \eqref{eqn:reversibility}; if $g(x_t|x')=g(x'|x_t)$ does not hold, the chain is not reversible in the~\eqref{eqn:reversibility} sense. Then the MH correction to the chain adds an additional step where the proposal $x'$ is accepted with probability 
\begin{align*}
\min \left\{1, \frac{f(x')g(x_t|x')}{f(x_t)g(x'|x_t)}\right\}\,.
\end{align*}
The following classical result confirms that the Markov chain thus generated is time-reversible and its invariant measure has a density function that is proportional to $f(x)$, and we include the main argument for completeness. 
\begin{lem}
\label{lem:mh_reversibility}
The Markov chain generated above is reversible in the sense of~\eqref{eqn:reversibility}.
\end{lem}
\begin{proof}
The transition probability of the Markov chain generated has the following form, 
\begin{align*}
p(x'|x_t)= g(x'|x_t) \min \left\{1, \frac{f(x')g(x_t|x')}{f(x_t)g(x'|x_t)}\right\}
\end{align*}
for all $x' \neq x$, and $p(x_t|x_t)$ taking up all the remaining probabilities. Without loss of generality, assume $f(x')g(x_t|x')\le f(x_t)g(x'|x_t)$, we have,  $$f(x_t)p(x'|x_t)=f(x_t)g(x'|x_t)\frac{f(x')g(x_t|x')}{f(x_t)g(x'|x_t)}=f(x')g(x_t|x')=f(x')p(x_t|x').$$
\end{proof}

\subsection{Reversibility for HMC}

For the HMC setup developed here, the stationary distribution is known up to a constant, i.e. $\pi(dq)= C\distribution(q)dq$. 
%
%
In standard HMC with symmetric auxiliary $\bbg$, each step starts with a $q\in\bbQ$ and a random momentum is generated and Hamiltonian motion applied to the pair to take it to $Q \in \bbQ$ with probability $\bbg(\Pi^{-1}_q(Q))$. Following the standard MH rejection prescription, 
we accept $Q$ with probability
\begin{align}\label{defn:mh_hmc}
\min \left\{1, \frac{\distribution(Q)\bbg(\Pi^{-1}_Q(q))}{\distribution(q)\bbg(\Pi^{-1}_q(Q))}\right\}\,.
\end{align}
Since the auxiliary distribution $\bbg$ is symmetric and supposing that the Hamiltonian motion can be implemented exactly, we have
$\Pi^{-1}_Q(q)=-\measureRho_\bbP \circ {\mathcal{R}}(q, \Pi^{-1}_q(Q))$,
with $\measureRho_\bbP$ being the projection from $\bbQ\times \bbP$ onto $\bbP$. Furthermore, 
$\bbg(\Pi^{-1}_Q(q))=\bbg(-\Pi^{-1}_Q(q))$. Then, the Hamiltonian energy preserving property, i.e. $\distribution(q)\bbg(p) =\distribution(Q)\bbg(P)$ implies that the ratio in~\eqref{defn:mh_hmc} is identically always equal to $1$.
Thus, standard HMC implemented with exact Hamiltonian motion can ignore the MH correction step. 

When the motion is implemented using a numerical integrator, the 
map $\Pi_q$ is being approximated by 
${\tilde \Pi}_q$, rooted in the approximation ${\mathcal {\tilde R}}$ to the map ${\mathcal R}$. Then the 
standard HMC is augmented with the MH rejection step~\eqref{defn:mh_hmc} with the map $\Pi_q$ replaced with ${\tilde \Pi}_q$. 
As in the case of exact motion, the symmetry of the auxiliary will ensure that the identity 
${\tilde \Pi}^{-1}_Q(q)=-\measureRho_\bbP \circ {\mathcal {\tilde R}}(q, {\tilde \Pi}^{-1}_q(Q))$ continues to hold.
The acceptance ratio however will not identically be one since the integrator may not preserve Hamiltonian energy, but Lemma~\ref{lem:mh_reversibility} leads the MH-rejection augmented implementation to converge to $\pi$.   

\subsection{Reversibility for AD-HMC}


One step of the proposed AD-HMC algorithm for asymmetrical auxiliaries $\bbg$  
starts from a $q_0\in \bbQ$ by generating a sample $p_0\in \bbP$ and applying forward Hamiltonian motion that then carries the pair $(q_0, p_0)$ to some $(q_1,p_{01})$. Then, another momentum $p_{12}\in \bbP$ is sampled and the backward Hamiltonian motion carries $(q_1,p_{12})$ to $(q_2, p_2)$, yielding the candidate $q_2$ for the next state. Similarly, should we start AD-HMC with $q_2$, the pair of momentum vectors $p_2$ and $p_{01}$ will take us back to $q_0$ through $q_1$. 
So, we will accept the proposed move to $q_2$ with probability
\begin{align} \label{defn:mh_adhmc}
\cP(q_0,q_1,q_2)=\min \left\{1, \frac{\distribution(q_2)\bbg(\Pi_{q_1}^{-b}(q_0))\bbg(\Pi_{q_2}^{-f}(q_1))}{\distribution(q_0)\bbg(\Pi_{q_0}^{-f}(q_1))\bbg(\Pi_{q_1}^{-b}(q_2))}\right\}\,,
\end{align}
where $\Pi^{-f}_{q_0}(q_1)$ denotes the momentum of the \emph{forward motion} that carries $q_0$ to $q_1$, and $\Pi_{q_1}^{-b}(q_2)$ the momentum of the \emph{backward motion} then carries $q_1$ to $q_2$. 

The transition probability of the AD-HMC Markov chain with the Hamiltonian motion augmented with the MH rejection step using~\eqref{defn:mh_adhmc} is equal to
$
\pr(q_0, q_2)=
\int_{\bbQ}\cP(q_0,q_1,q_2)\bbg(\Pi_{q_0}^{-f}(q_1))\bbg(\Pi_{q_1}^{-b}(q_2))\,dq_1
$.
In Lemma~\ref{lem:adhmc_mh_reversibility}, we first establish that this Markov chain produces the desired time reversibility of the augmented AD-HMC procedure.

\begin{lem}
\label{lem:adhmc_mh_reversibility}
For any $q_0$ and $q_2$ in $\bbQ$, the augmented AD-HMC algorithm satisfies the reversibility condition  $\bbf(q_0)\pr(q_0, q_2)= \bbf(q_2) \pr(q_2, q_0)$. 
\end{lem}
\begin{proof}
The transition from $q_0$ to $q_2$ may happen through any $q_1$ therefore its probability density 
is
\begin{align*}
&\pr(q_0, q_2)=\int_{\bbQ} \min \left\{1, \frac{\distribution(q_2)\bbg(\Pi_{q_1}^{-b}(q_0))\bbg(\Pi_{q_2}^{-f}(q_1))}{\distribution(q_0)\bbg(\Pi_{q_0}^{-f}(q_1))\bbg(\Pi_{q_1}^{-b}(q_2))}\right\}\cdot\bbg(\Pi_{q_0}^{-f}(q_1))\bbg(\Pi_{q_1}^{-b}(q_2))\,dq_1
\\
&=
\frac{1}{\distribution(q_0)}\int_{\bbQ} \min \left\{\distribution(q_0)\cdot\bbg(\Pi_{q_0}^{-f}(q_1))\bbg(\Pi_{q_1}^{-b}(q_2)), \distribution(q_2)\cdot\bbg(\Pi_{q_1}^{-b}(q_0))\bbg(\Pi_{q_2}^{-f}(q_1))
\right\}dq_1\,.
\end{align*}
Because the expression under the integral is symmetric with respect to $q_0$ and $q_2$ we conclude that 
$ \distribution(q_0)\pr(q_0,q_2)=\distribution(q_2)\pr(q_2,q_0)$.
\end{proof}

The transition acceptance probability proposed in~\eqref{defn:mh_adhmc} and analysed in Lemma~\ref{lem:adhmc_mh_reversibility} inherently assumes that the implementation of the motion maps $\Pi^{-f}$ and $\Pi^{-b}$ are themselves reversible, in that starting from a pair $(q,p)$ and applying forward motion leads it to the pair $(q',p')$, where applying backward motion to $(q',p')$ will return us to $(q,p)$. This is satisfied if the Hamiltonian motion implementation is exact and is also satisfied by well-designed numerical intergration approximators such as the leapfrog symplectic integrator \cite{verlet67}.  
Moreover, for exact Hamiltonian motion we also have that the Hamiltonian energy is preserved in both the forward and backward motion, which is easily verified to lead to the acceptance probability in~\eqref{defn:mh_adhmc} to be always $1$.  

\subsection{Geometric Convergence of Leapfrog AD-HMC}

To demonstrate the geometric convergence of AD-HMC with the Leapfrog or St\"ormer–Verlet first-order numerical integration method~\cite{verlet67}, we will first provide some estimations on the motion estimator. 
The Leapfrog integrator
interleaves discrete momentum ($p$) and position ($q$) variable updates in a sequence of steps where only one variable changes at a time.
The method takes $L$ steps each of size $\epsilon$ to move a path of length $T=L\epsilon$. 
For the ease of exposition, we present the properties of the estimator for one step, i.e. $L=1$, but they can be applied to the case of general $L$ steps. Each step of the Leapfrog integrator takes a pair of position and momentum $(Q(0),P(0))$ and updates it to $(Q(\ep),P(\ep))$ though the following calculations: 
\begin{align}
	\label{leapfrog_1}
P\left({\frac {\epsilon} 2 }\right) &= P(0) - \frac {\epsilon}{2}\,\, \frac{\partial H}{\partial q}(Q(0),P(0))\,,\\ \label{leapfrog_2}
Q(\epsilon) &= Q(0) + \epsilon \,\,\frac{\partial H}{\partial p}\left(Q(0),P\left(\frac {\epsilon} 2 \right)\right)\,,\,\,\text{and}\\ \label{leapfrog_3}
P (\epsilon) &= P\left(\frac {\epsilon} 2\right) 
- \frac {\epsilon} 2 \,\,\frac{\partial H}{\partial q}(Q(\epsilon) ,P(\frac {\epsilon} 2 ))\,.
\end{align}  
\begin{lem}
Under the strong convexity assumptions on $U$ and $V$, there exists $\ep>0$ and $\beta_L\in(0,1)$ such that for any $q_1, q_2\in \bbQ$, we have that 
$W_1({\tilde Q}_1, {\tilde Q}_2)< \beta_L \distance( q_1, q_2)$, 
with ${\tilde Q}_1, {\tilde Q}_2$ denoting the position vectors obtained through Leapfrog integration with step size $\ep$. 
\label{lem:leapfrog_Q_contraction}
\end{lem}
\begin{proof}
First, by the definition of the Wasserstein distance, we know that
$W_1({\tilde Q}_1, {\tilde Q}_2)\le \distance({\tilde q}_1, {\tilde q}_2)$, 
with ${\tilde q}_1, {\tilde q}_2$ denoting the position vectors obtained through Leapfrog integration starting from $q_1$ and $q_2$ with the same momentum vector $p$. 
Meanwhile, for each $i=1,2$, from \eqref{leapfrog_1}, \eqref{leapfrog_2} and \eqref{leapfrog_3}, and the assumption on the boundedness of derivatives of $U(\cdot)$ and $V(\cdot)$, we have, uniformly over all $q_i\in \bbQ$, 
\begin{align*}
{\tilde q_i} =& q_i + \epsilon \nabla V(p) -\frac{\epsilon^2}{2} \nabla^2 V(p) \nabla U(q_i) + O(\epsilon^3)
\end{align*}
Therefore, 
\begin{align*}
\distance ({\tilde q_1},  {\tilde q_2})\le \distance ( q_1, q_2) \left[1-\frac{\epsilon^2}{2} 
\nabla^2 V(p) M +O(\epsilon^3)\right],
\end{align*}
since $\distance (\nabla U(q_1), \nabla U(q_2)) \ge M \distance (q_1, q_2)$ following from the assumption that both $\nabla^2 U$ and $\nabla^2 V$ are bounded from above and away from zero. The result follows. 
\end{proof}
Note that the proof of Lemma~\ref{lem:leapfrog_Q_contraction} analyses Leapfrog steps using a common momentum vector $p$, which allow us to get better estimation than those in~\cite{Durmus2017OnTC}. We next analyse the distance between the corresponding momentum vectors.
\begin{lem}
Under the strong convexity assumptions on $U$ and $V$, there exists a $\beta_L\in(0,1)$ such that for any $q_1, q_2\in \bbQ$, we have that
$W_1({\tilde P}_1, {\tilde P}_2)< \beta_L \distance( q_1, q_2)$.
\label{lem:leapfrog_P_contraction}
\end{lem}
\begin{proof}
Again, leapfrog algorithm gives us,
\begin{align*}
P (\epsilon)&= P (\frac{\epsilon}{2})-\frac{\ep}{2}\nabla U(Q (\epsilon))\\
&=P(0)-\frac{\ep}{2} \nabla U(Q(0)) -\frac{\ep}{2}\nabla U(Q(0)+\ep \nabla V(P(0))+O(\ep))
\\=& P(0)-\ep\nabla U(Q(0)) +O(\epsilon^2) 
\end{align*}
The contractions following from the assumption that both $\nabla^2 U$ and $\nabla^2 V$ are bounded from above and away from zero. 
\end{proof}

These lemmas lead to our final result that the Leapfrog implemented AH-HMC algorithm, augmented with the rejection step~\eqref{defn:mh_adhmc}, exhibits geometric convergence.


\begin{thm}
\label{thm:three}
If $ \Big| \frac{\partial Q}{\partial q} \Big|\le \beta <1$ outside $B_R=\{q:\distance(q, 0) \le R\}$ for some $R>0$, and bounded within $B_R$, the Hamiltonian integration is implemented with a symplectic leapfrog integrator, then the recursion generated by the AD-HMC procedure augmented with the MH-step~\eqref{defn:mh_adhmc} converges geometrically in the $W_1$ Wasserstein metric.
\end{thm}
\begin{proof}
The proof follows the same idea as that of Theorem \ref{thm:two}, and we show that for any pair $(q_1, q_2)$, there exists a $\beta'\in(0,1)$ such that $W_1(\mu(q_1), \mu(q_2)) \le \beta' \distance(q_1, q_2)$. It is straightforward to see that Lemma \ref{lem:smallset} still holds; 
the contraction of the Wasserstein distance for $(q_1, q_2)  \in C\times C$ also follows from Lemma \ref{lem:leapfrog_Q_contraction}.

Now, let us look at bivariate drift condition. For any two starting positions $(q_0, {\tilde q}_0)  \notin C\times C$ and the same two momentum vectors $p_0\in \bbP$ and $p_{12}\in \bbP$, the leapfrog integrator will produce $(q_1,p_{01})$ and $(q_2,p_2)$ from $(q_0, p_0)$, and  $({\tilde q}_1,{\tilde p}_{01})$ and $({\tilde q}_2,{\tilde p}_2)$ from $({\tilde q}_0, p_0)$. 
The proposed moves to $q_2$ and $\tilde{q}_2$ respectively will be accepted with probabilities $\xi_1$ and $\xi_2$ respectively, with
\begin{align*}
\xi_1 := \min \left\{1, \frac{\distribution(q_2)\bbg({\tilde \Pi}_{q_1}^{-b}(q_0))\bbg({\tilde \Pi}_{q_2}^{-f}(q_1))}{\distribution(q_0)\bbg({\tilde \Pi}_{q_0}^{-f}(q_1))\bbg({\tilde \Pi}_{q_1}^{-b}(q_2))}\right\}\,,
\end{align*}
\begin{align*}
\xi_2 := \min \left\{1, \frac{\distribution({\tilde q}_2)\bbg({\tilde \Pi}_{{\tilde q}_1}^{-b}({\tilde q}_0))\bbg({\tilde \Pi}_{q_2}^{-f}({\tilde q}_1))}{\distribution({\tilde q}_0)\bbg({\tilde \Pi}_{{\tilde q}_0}^{-f}({\tilde q}_1))\bbg({\tilde \Pi}_{{\tilde q}_1}^{-b}({\tilde q}_2))}\right\}\,.
\end{align*}
Now, consider
\begin{itemize}
\item 
In this case, both moves are rejected, the Wasserstein distance (conditioned on this event) will remain to be $\distance (q_0, {\tilde q}_0)$. 
\item
$\xi\le \xi_1$ and $\xi \le \xi_2$. Then Lemma \ref{lem:leapfrog_Q_contraction} applies. 
\item
$\xi\le \xi_1$ and $\xi >\xi_2$.
The probability of this happening is equal to $\xi_1-\xi_2$. Naturally, in this case, $\xi_2<1$. We need to establish that there exists a uniform constant $\varrho \in (0,1)$, such that $\xi_1-\xi_2 \le \varrho \ep \distance (q_0, {\tilde q}_0)$, with $\ep$ being the step parameter in the leapfrog algorithm. Evidently,
\begin{align*}
\xi_1-\xi_2\le \Big|\frac{\distribution(q_2)\bbg({\tilde \Pi}_{q_1}^{-b}(q_0))\bbg({\tilde \Pi}_{q_2}^{-f}(q_1))}{\distribution(q_0)\bbg({\tilde \Pi}_{q_0}^{-f}(q_1))\bbg({\tilde \Pi}_{q_1}^{-b}(q_2))}-\frac{\distribution({\tilde q}_2)\bbg({\tilde \Pi}_{{\tilde q}_1}^{-b}({\tilde q}_0))\bbg({\tilde \Pi}_{q_2}^{-f}({\tilde q}_1))}{\distribution({\tilde q}_0)\bbg({\tilde \Pi}_{{\tilde q}_0}^{-f}({\tilde q}_1))\bbg({\tilde \Pi}_{{\tilde q}_1}^{-b}({\tilde q}_2))}\Big|
\end{align*}
The right hand side can be bounded by the summation of the following two terms,
\begin{align*}
I_1=\Big|\frac{\distribution(q_2)}{\distribution(q_0)}-\frac{\distribution({\tilde q}_2)}{\distribution({\tilde q}_0)}\Big|\frac{\bbg({\tilde \Pi}_{q_1}^{-b}(q_0))\bbg({\tilde \Pi}_{q_2}^{-f}(q_1))}{\bbg({\tilde \Pi}_{q_0}^{-f}(q_1))\bbg({\tilde \Pi}_{q_1}^{-b}(q_2))}
\end{align*}
and
\begin{align*}
I_2=\Big|\frac{\bbg({\tilde \Pi}_{q_1}^{-b}(q_0))\bbg({\tilde \Pi}_{q_2}^{-f}(q_1))}{\bbg({\tilde \Pi}_{q_0}^{-f}(q_1))\bbg({\tilde \Pi}_{q_1}^{-b}(q_2))}-\frac{\bbg({\tilde \Pi}_{{\tilde q}_1}^{-b}({\tilde q}_0))\bbg({\tilde \Pi}_{q_2}^{-f}({\tilde q}_1))}{\bbg({\tilde \Pi}_{{\tilde q}_0}^{-f}({\tilde q}_1))\bbg({\tilde \Pi}_{{\tilde q}_1}^{-b}({\tilde q}_2))}\Big|\frac{\distribution({\tilde q}_2)}{\distribution({\tilde q}_0)}.
\end{align*}
Consider,
\begin{align*}
\Big|\frac{\distribution(q_2)}{\distribution(q_0)}-\frac{\distribution({\tilde q}_2)}{\distribution({\tilde q}_0)}\Big|=&\frac{\distribution(q_2)}{\distribution(q_0)}\Big|1-\frac{\distribution(q_0)}{\distribution(q_2)}\frac{\distribution({\tilde q}_2)}{\distribution({\tilde q}_0)}\Big|\\=&\frac{\distribution(q_2)}{\distribution(q_0)}|1-\exp[U(q_0)+U({\tilde q}_0)-U(q_2)-U({\tilde q}_2)]|
\end{align*}
Applying the argument in the proof of Lemma \ref{lem:leapfrog_Q_contraction} twice, we can see that, $\left|\frac{\distribution(q_2)}{\distribution(q_0)}\right|$ can be uniformly bounded due to our strong convexity assumption. Furthermore, we have
\begin{align*}
 |U(q_0)+U({\tilde q}_0)-U(q_2)-U({\tilde q}_2)|= &[\frac{\ep^2}{2}+O(\ep^3)] \distance (q_0, {\tilde q}_0)
\end{align*} 
Meanwhile, there is a uniform bound of 
\begin{align*}
\frac{\bbg({\tilde \Pi}_{q_1}^{-b}(q_0))\bbg({\tilde \Pi}_{q_2}^{-f}(q_1))}{\bbg({\tilde \Pi}_{q_0}^{-f}(q_1))\bbg({\tilde \Pi}_{q_1}^{-b}(q_2))},
\end{align*}
again due to the strong convexity assumption. Thus, we can conclude $I_1=\varrho [\ep^2+O(\ep^3)] \distance (q_0, {\tilde q}_0)$. 
Similar logic can be applied to $I_2$. Therefore, the required estimate. 
\item
$\xi>\xi_1$ and $\xi \le \xi_2$.
The treatment is the same as the previous case. 
\end{itemize}
In summary, if the acceptance/rejection decisions are the same for the exact and numerical solution, the closeness of the numerical integrator, from Lemma \ref{lem:leapfrog_Q_contraction}, ensures that the drift condition will be satisfied; meanwhile, Lemmas \ref{lem:leapfrog_Q_contraction} and \ref{lem:leapfrog_P_contraction} can also help to bound that the probability of the scenario that the acceptance/rejection decisions are different such that the drift condition can also be satisfied, as demonstrated in the last two cases above. 

Hence, there is a proper $\beta'<1$ such that, $W_1(\mu(q_0), \mu({\tilde q}_0))\le \beta'\distance (q_0, {\tilde q}_0)$, which implies the desired geometric convergence in Wasserstein distance.  
\end{proof}


\section{Adaptive Algorithm} \label{sec:algo}
The single iterate update step of the Alternating Direction HMC (AD-HMC) method is presented in Algorithm~\ref{algo:adhmc} and is laid side-by-side to the standard HMC method (Algorithm~\ref{algo:hmc}) to elicit their differences. Two key differences are notable. First, an iteration of the AD-HMC method consists of two implementations of the combination of \emph{sampling-and-lifting, applying Hamiltonian motion and projecting} that constitute an iteration of the standard HMC. In the first implementation, the motion is in the forward (i.e. for length $+T$) direction, and in the second implementation, the motion is in the backward (length $-T$) direction. 
%
%
The second difference is in the calculation of the MH acceptance/rejection probability, which is adjusted accordingly. As observed in Theorems~\ref{thm:strong conv} and~\ref{thm:three}, the 
operator $\cT_A$ corresponding to the AD-HMC algorithm enjoys invariance properties and is self-adjoint even for asymmetric momentum distributions $\bbg(p)$, thus ensuring convergence.     

The main practical advantage provided by the AD-HMC method is in being able to utilize arbitrary distributions as the auxiliary. In particular, this allows for the freedom of {\em adaptively designing} the auxiliary in order to accelerate convergence to the target distribution.
%
Towards this, 
we consider a heuristic that adaptively constructs a Gaussian mixture as the auxiliary. Our approach 
attempts to match the kinetic energy $V$ of the constructed auxiliary $\bbg$ to the potential $U$ of the desired target $\distribution$. Recall the solution to the AD-HMC evolution equation in Proposition~\ref{prop:d QP d qp} and note that -- speaking loosely -- if $V$ resembles $U$, then the matrices $A$ and $B$ also resemble each other and the evolution matrix resembles a (generalized) rotation. Thus, this adaptive AD-HMC 
heuristic generalizes the Riemannian-Gaussian procedure~\citep{GirolamiCalderhead11} and stands to enjoy the same benefits of regularized and easier-to-approximate Hamiltonian motion that leads to faster convergence in practice.

\begin{minipage}[t]{0.46\linewidth}
\vskip -0.2in
\begin{algorithm}[H]
\begin{algorithmic}
 \STATE \textbf{Initialization:} potential energy $U(q)$, kinetic energy $V(q)$, initial iterate 
 $q_{0}$, 
 trajectory length $T$\;
  \STATE Sample: $p_{0}\sim \bbg(p)$  \;  
  \STATE Lift: $(q_{0}, p_0) \leftarrow q_{0}$ \;
  \STATE 
  Hamiltonian motion for $+T$:\;
  \STATE $\;\; (Q(T), P(T)) \leftarrow (q_{0}, p_0)$\;
  \STATE Project: $q_{1} \leftarrow Q(T)$ \;
\STATE MH Rejection:\; 
\STATE \;\;Sample $Z\sim U(0,1)$\; 
\IF{ $Z \le 
\frac{\distribution(q_1)\bbg(P(T))}{\distribution(q_0)\bbg(p_0)}\,
$}
  \STATE Return $q_1$
  \ELSE
  \STATE Return $q_0$ 
  \ENDIF
\end{algorithmic}
\caption{Standard HMC}
\label{algo:hmc}
\end{algorithm}
\end{minipage}
\hfill
\begin{minipage}[t]{0.46\linewidth}
\vskip -0.2in
\begin{algorithm}[H]
\begin{algorithmic}
\STATE \textbf{Initialization:} potential energy $U(q)$, kinetic energy $V(q)$, initial iterate 
$q_{0}$ 
, trajectory length $T$
\STATE ({\it Forward motion})
\STATE  \;\;Sample: $p_{0}\sim \bbg(p) \qquad$ 
\STATE  \;\;Lift: $(q_{0}, p_0) \leftarrow q_{0}$ \;
\STATE  \;\;Hamiltonian motion for $+T$:
\STATE  \;\;\;\;$ (Q(T), P(T)) \leftarrow (q_{0}, p_0)$\;
\STATE  \;\;Project: $q_{0'} \leftarrow Q(T)$ \;
\STATE ({\it Backward motion})
\STATE  \;\;Sample: $p_{0'}\sim \bbg(p)\quad$ 
\STATE  \;\;Lift: $(q_{0'},p_{0'}) \leftarrow q_{0'}$ \;
\STATE  \;\;Hamiltonian motion for $-T$:
\STATE  {\small $(Q(-T), P(-T))\leftarrow (q_{0'}, p_{0'})$}\;
\STATE  \;\;Project: $q_{1} \leftarrow Q(-T)$ \;
\STATE MH Rejection:
\STATE \;\; Sample $Z\sim U(0,1)$\; 
\IF{ $Z \le 
\frac{\distribution(q_1)\bbg(P(T))\bbg(P(-T))}{\distribution(q_0)\bbg(p_0)\bbg(p_{0'})}\,
$}
  \STATE Return $q_1$
  \ELSE
  \STATE Return $q_0$ 
  \ENDIF
\end{algorithmic}
\caption{AD-HMC}\label{algo:adhmc}
\end{algorithm}
\vskip -.5in
\end{minipage}

The change in the auxiliary needs to be introduced with care in order to preserve the Markovian nature of the iterates. For instance, \cite{gelfand94sahu} provide an example of a self-tuning MCMC that periodically adapts its kernel; while each kernel has the same limiting distribution, the adaptation step introduces significant path dependence on past history and alters the limit distribution. The Adaptive Alternating Direction HMC method presented in Algorithm~\ref{algo:adapt_adhmc} eliminates this concern by updating the auxiliary at special points in the run-length called \emph{regeneration} times. 
This procedure is adopted from the adaptive MCMC scheme studied in~\cite{brockwell05kadane}. The regeneration times are sampled via rejection with respect to the ratio of an additional auxiliary distribution $\phi$ and the target distribution. The auxiliary $\bbg$ may be updated only when the Markov Chain enters regeneration. The chain also exits the regeneration state via acceptance sampling using the (reverse of) the same ratio, and the next iteration upon exit is a sample from $\phi$ independent of the past. The control parameter $c_{\phi}$ controls the frequency with which regeneration is entered.

The adaptation of $\bbg$ is permitted at points of regeneration using all the information in the previous regeneration tours. In addition, the regeneration parameters $\phi$ and $c_{\phi}$ may also be changed. Note that the ideal choice for $\phi$ is $\distribution$ itself, along with $c_\phi$ set to $\distribution$'s normalizing constant since this makes Algorithm~\ref{algo:adapt_adhmc} enter regeneration at every iteration and produces an independent sample from $\distribution$ as the output. However, $\distribution$ cannot be sampled directly, but note that the target of adapting $\phi$ are the same as $\bbg$. Thus, we set both $\bbg$ and $\phi$ to the same updated Gaussian mixture model approximated over all the previous data utilizing Algorithm~\ref{algo:cluster} (see details in Sec.~\ref{sec:expts}). The parameter $c_\phi$ is estimated as the average value of the ratio $\distribution(\cdot) / \phi(\cdot)$ over all the previous iterates in order to yield that $(c_\phi \; \phi(\cdot)) / \distribution(\cdot) \approx 1$. 

\begin{algorithm}[H]
\begin{algorithmic}
 \STATE \textbf{Initialization:} potential energy $U(q)$, kinetic energy $V(q)$, initial iterate 
 $q_{0}$, trajectory length $T$,\; initial auxiliary $\bbg(p)$,\; regeneration distribution $\phi(q)$,\; regeneration multiplier $c_{\phi}$
  \STATE Sample $q_{0'}$ from Alternating Direction HMC starting at $q_0$ 
  \STATE Sample $Z\sim U(0,1)$ 
\IF{ $Z \ge 
\frac {c_{\phi}\phi(q_{0'})}{\distribution(q_{0'})}
$}
  \STATE Return $q_{0'}$ 
  \ELSE 
  \STATE {\it(regeneration)}
\STATE Adapt auxiliary $\bbg$,\; regeneration distribution $\phi$ and multiplier $c_\phi$
\REPEAT
\STATE Sample $q_{0''} \sim \phi$ and a $Z' \sim U(0,1)$
\UNTIL{ $Z' \le \frac {\distribution(q_{0''})} {c_{\phi}\phi(q_{0''})}$}
  \STATE Return $q_{0''}$
  \ENDIF
\end{algorithmic}
\caption{Adaptive Alternating Direction HMC}
\label{algo:adapt_adhmc}
\end{algorithm}
%

{Algorithm~\ref{algo:adapt_adhmc} produces estimates of averages that are statistically consistent with the target $\distribution$. While a complete analysis of the adaptive scheme is out of scope for this article, we briefly describe how this works in the following sense. Let  
$m_u = \lim_{T\rightarrow \infty} \frac 1 T \sum_t u(q_t)$ be a long run average of a function $u(\cdot)$ using the iterates $q_t$ of Algorithm~\ref{algo:adapt_adhmc}. 
Then, by the invariance of $\distribution$, we have that $m_u \;\propto\; \int_\bbQ u(q) d \distribution(q) $, where the proportionality constant is unknown. Moreover, with 
a sequence of regenerative stopping times $\tau$, by the Optional Stopping Theorem for Martingales~\citep[Sec. 5.7]{durrett2019probability}, 
we get that 
$ m_u = \frac 1 {\ex [\tau]} \ex \left[\sum_{t=0}^{\tau} u(q_t)\right]$.
Thus, for any sequence of regeneration times $\tau_i,\,\,i=1,2,\ldots$, the pairs $\left( M_i = \sum_{t=s_{i-1}+1}^{s_i} u(q_t) ,  \tau_i\right)$ 
can be used to estimate $m_u$. 
%
\cite{gilksRoberts98Sahu} show that the pairs $(M_i, \tau_i)$ may be generated from different dynamics -- in our case different $(\bbg, \phi, c_\phi)$ -- and a consistent estimator for $m_u$ can be constructed from such pairs as long as $\distribution$ is the invariant of each such run, which is ensured by the convergence guarantees of the AD-HMC procedure. \cite{gilksRoberts98Sahu} also provide a central limit theorem for such an estimator and illustrate how to approximate the (finite) variance of the estimator. }

\section{Numerical Experiments}\label{sec:expts}

In this section, we study the efficacy of using an asymmetric momentum distribution $\bbg(p)$ for Hamiltonian Monte Carlo using AD-HMC in comparison to using Gaussian distributions in standard HMC. 
Algorithm~\ref{algo:adapt_adhmc} proposed in Sec.~\ref{sec:algo} adaptively constructs an asymmetric auxiliary to speed convergence of AD-HMC to $\distribution$. 


We study a global optimization problem of finding the minima of an objective function in $\Real^3$ with seven distinct local optimal solutions with well-separated zones of attraction. As noted earlier, recent literature in this area~\cite{chau22sghmc,gao21sghmc} seeks to leverage the geometric convergence properties of the HMC method to provide strong finite time guarantees for finding global optimal solutions. This is done by employing the objective function as the potential function $U(q)$ in an HMC scheme to produce samples that tend to concentrate near local minima of $U(q)$ (equivalently, high-probability regions of $\distribution(q)$), and then pruning this list of minima to choose the global minima.
A representation of the resulting target distribution is displayed in \textbf{Fig.}~\ref{fig:helix_hmcs} (left), where the seven local optima are placed along a helical curve and the $\distribution(q)$ is assembled as an equal-probability mixture of uncorrelated Gaussians of different standard deviations ranging from $[0.1, 0.7]$ centered on these local optima. The two highly concentrated masses in the central part have the steepest $U(q)$ valleys and also represent the global optima.


We include the basic HMC Algorithm (Alg.~\ref{algo:hmc}) with symmetric Gaussians as auxiliaries. 
Let $N(\mu,\Sigma )$ define a Gaussian with mean $\mu$ and covariance matrix $\Sigma$. This makes the kinetic energy $V(p)$ symmetric and quadratic: $V(p) \propto (p-\mu)^t\Sigma^{-1}(p-\mu)$. The \texttt{`hmc-StdG'} case uses $\bbg(p) \sim N(0, I) $ to represent the standard non-adaptive approach. 

The dynamically varying symmetric Gaussian auxiliary heuristic described by~\cite{GirolamiCalderhead11} is denoted the \texttt{`hmc-RG'} scheme in our experiments below, where 
the covariance matrix $\Sigma$ is updated at regeneration points in a manner similar to Algorithm~\ref{algo:adapt_adhmc} based on the local covariance matrix of the 
most recently visited samples. The intent of this procedure is to use the local Hessian $\frac{\partial^2 U}{\partial q^2}$ as the covariance. This is however an expensive computation and so an additional approximation that the density is locally Gaussian
is deployed, which allows the Hessian to be estimated using the covariance of the most recent iterates.

The \texttt{`adhmc'} heuristic maintains a Gaussian mixture auxiliary constructed to match all the modes of $\distribution(q)$ where recent samples have concentrated. Formally, it starts with $\bbg_0(p) \sim N(0,I)$ and at each regeneration point adapts $\bbg$ according to Algorithm~\ref{algo:cluster}. Recall from Thm.~\ref{thm:three} that the AD-HMC algorithm run with each asymmetric auxiliary constructed by \texttt{`adhmc'} converges geometrically. 
\begin{algorithm}[H]
\begin{algorithmic}
 \STATE \textbf{Initialization:} Collection of past iterates $\{q_t,\;t=1,\ldots,T\}$
 \STATE $\mathbf{1.}\,\,\,$ Apply labels $c=-1,0,\ldots,C$ to samples $q_t$ using the OPTICS clustering algorithm~\citep{opticscluster}, which dynamically determines both the number of clusters $C$ and the assignment of labels; the samples in the $c=-1$ class denote those that were not classified. 
\STATE $\mathbf{2.}\,\,\,$ For each $c$, let $s_c$ represent number of $q_t$ labeled $c$, with $T = \sum_{c=-1}^C s_c$.  
\STATE $\mathbf{3.}\,\,\,$ Estimate sample mean $\mu_c$ and covariance matrix $\Sigma_c$;\; set $\nu_c = \frac {s_c} T$.
\STATE $\mathbf{4.}\,\,\,$ Set auxiliary $\bbg(p) \sim \left\{ N(\mu_c, \Sigma_c)\right.$ w.p. $\left.\nu_c,\;\;c=-1,\ldots,C\,  \right\}$.
\STATE $\mathbf{5.}\,\,\,$ Set regeneration distribution $\phi\leftarrow \bbg\;\;\;$and$\;\;c_\phi \;\leftarrow\; \frac 1 T \sum_t \frac {\phi(q_t)} {\distribution(q_t)}.  $ 
\end{algorithmic}
\caption{Clustering Adaptation of Auxiliary $\bbg, \phi$ and $c_\phi$ }
\label{algo:cluster}
\end{algorithm}

The Hamiltonian motion dynamics in the algorithms are approximated by the leapfrog or St\"ormer–Verlet first-order symplectic numerical integration procedure~\citep{verlet67} described in Section~\ref{sec:perturbations}. 
The method takes $L$ steps each of size $\epsilon$ to move a path of length $T=L\epsilon$, where in each step the calculations given by~\eqref{leapfrog_1}-\eqref{leapfrog_3} are performed. The parameters $\epsilon$ and $L$ need to be chosen carefully so that the proportion of samples rejected by the MH rejection step with~\eqref{defn:mh_adhmc} probability is minimal, and in practice dynamically updated schemes such as the NUTS procedure~\citep{hoffman14} are popular. Our experiments fix the $\epsilon$ and $L$ values to present an uncluttered view of the comparison of the different methods.

All algorithms were implemented in python 3.7 and ran on a 
server with two AMD EPYC 7301 $16$-core processors and 64Gb system memory and two GeForce RTX2070 8Gb GPUs. The \texttt{`adhmc'} method runs the OPTICS clustering scheme with its default parameters in the Scipy library~\citep{scikit-learn}. Note that the clustering step limits the use of the adaptive schemes to only moderate dimensional settings.


The three methods are primarily compared using their performance in reducing the distance between the iterate distribution and the desired target distribution. The iterate distribution is represented by the empirical distribution of the most recently visited samples. 
The $W_1$ distance is numerically approximated using the Sinkhorn method~\cite{feydy2019interpolating} to give a flavor of the geometric convergence results that are at the heart of this paper. 

\textbf{Fig.}~\ref{fig:helix_hmcs} (right) provides a comparison of the performance of the three HMC methods in reducing the $W_1$ distance between the iterates and the target distribution as a function of the total wallclock hours elapsed in the top panel, while the bottom panel presents the average per-iteration CPU seconds taken by each method. The first half hour of iterations are discarded as the initialization phase.  The $W_1$ distance is computed using the empirical distributions obtained with the last $900$ visited points. Each result is calculated by aggregating data from $50$ independent replications of each method, with the individual replications plotted with a lighter shade.
The \texttt{`hmc-RG'} and \texttt{`adhmc'} adaptive methods use the last visited $2000$ points in $\bbQ$ for their modification of the auxiliary distribution, and $n_a$ is set to allow for $10$ modification steps throughout the run of each experiment. 

Recall that the AD-HMC takes two sets of Hamiltonian motion steps per iterations, and consequently the \texttt{`adhmc'} steps start out twice as expensive as the other two methods. As the iterations progress, the multi-modal asymmetric Gaussian mixture distribution computed by \texttt{`adhmc'} adds to their time-complexity. However, the averaged distance of the iterates of the \texttt{`adhmc'} method fall much faster than the \texttt{`hmc-StdG'} and \texttt{`hmc-RG'} methods. Also note that the linear drop in the $W_1$ distance in log-scale, which indicates the geometric convergence of the three methods.
\begin{figure}[tb]
	\centering
	\includegraphics[width=0.47\linewidth]
	{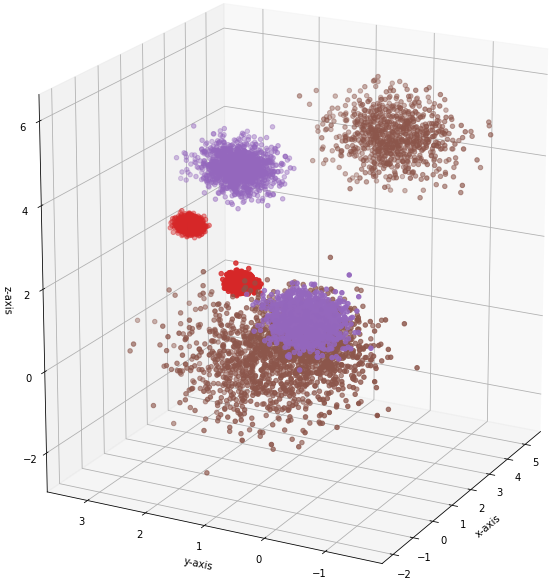}
	\hskip 0.2in
	\includegraphics[width=0.47\linewidth]
	{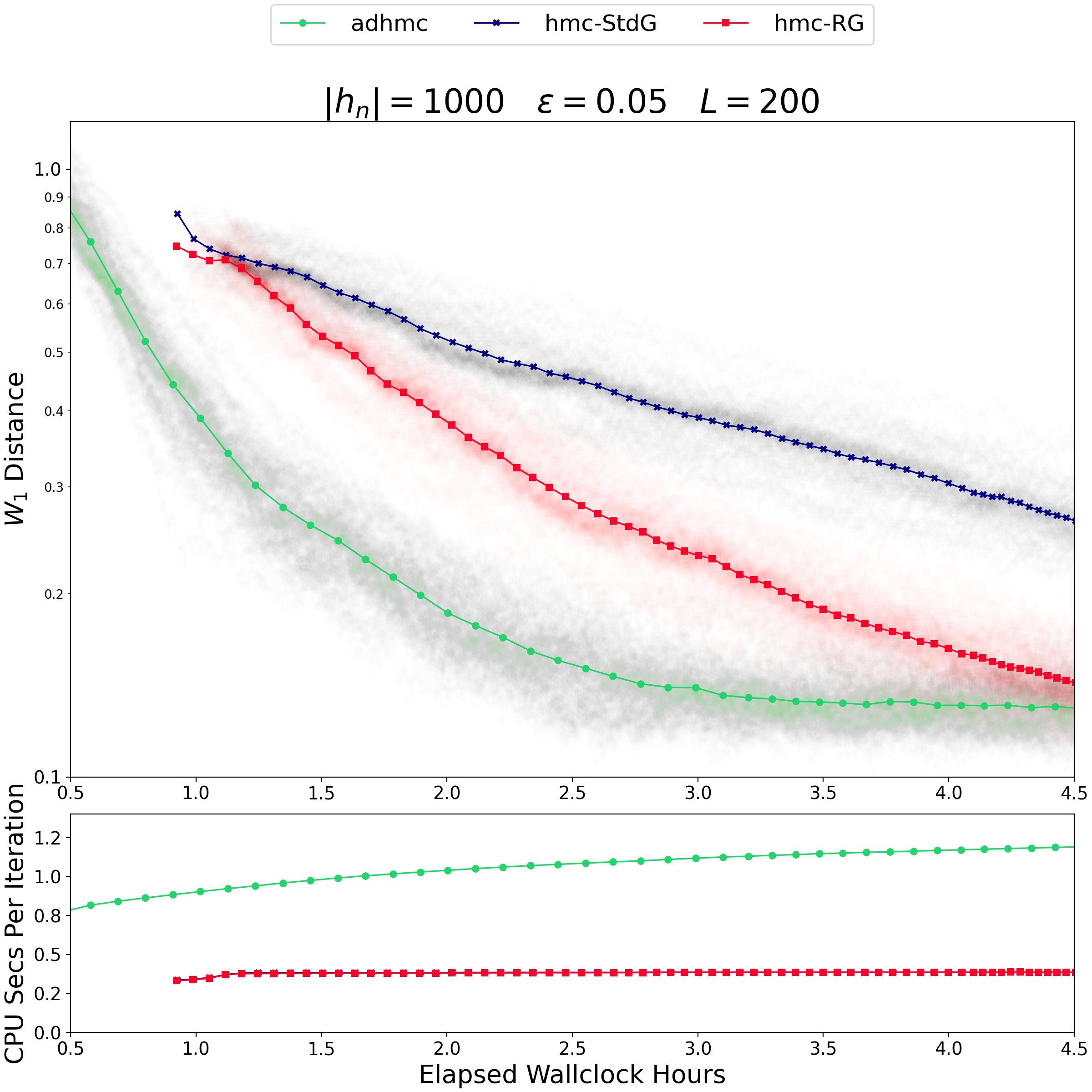}
	\caption{(left) The target distribution with six local optima, where the smaller variance around the optimal point indicates a corresponding lower optimal value. (right) Comparison of the AD-HMC, HMC with standard Gaussian and HMC with STAN-like adaptive Gaussian methods  (top) $W_1$ distance (in log-scale) between iterate and target distribution and (bottom) average per-iteration computation time in seconds; both plots use CPU wallclock time in hours as the $x$-axis, and fix the number of particles $|h_n|= 1000$ in each iteration, with Leapfrog integrator parameters $\epsilon=0.05$ and $L=200$.}
	\label{fig:helix_hmcs}
\end{figure}

\tbfg~\ref{fig:iteratestrip} presents the iterates  returned at the termination of the three algorithms studied in the results in~\tbfg~\ref{fig:helix_hmcs} (right). The iteration is represented by a collection of $1000$ particles, each initiated at the origin. 
The \texttt{`adhmc'} method is seen to have identified all the nodes of the target, and samples are also concentrated as expected from the form of $\distribution$ with both the sharp narrow nodes (colored in red) given equal weight. The approximations of the  \texttt{`hmc-StdG'}'s $h_n$ are still far from the target at termination, where one of the two global optima is barely identified. While the \texttt{`hmc-RG'} method is able to identify the two concentration points (global optima), the sample weights to each node is imbalanced compared to the true $\distribution$.

\begin{figure}[tb]
	\centering
    \includegraphics[width=0.3\linewidth]{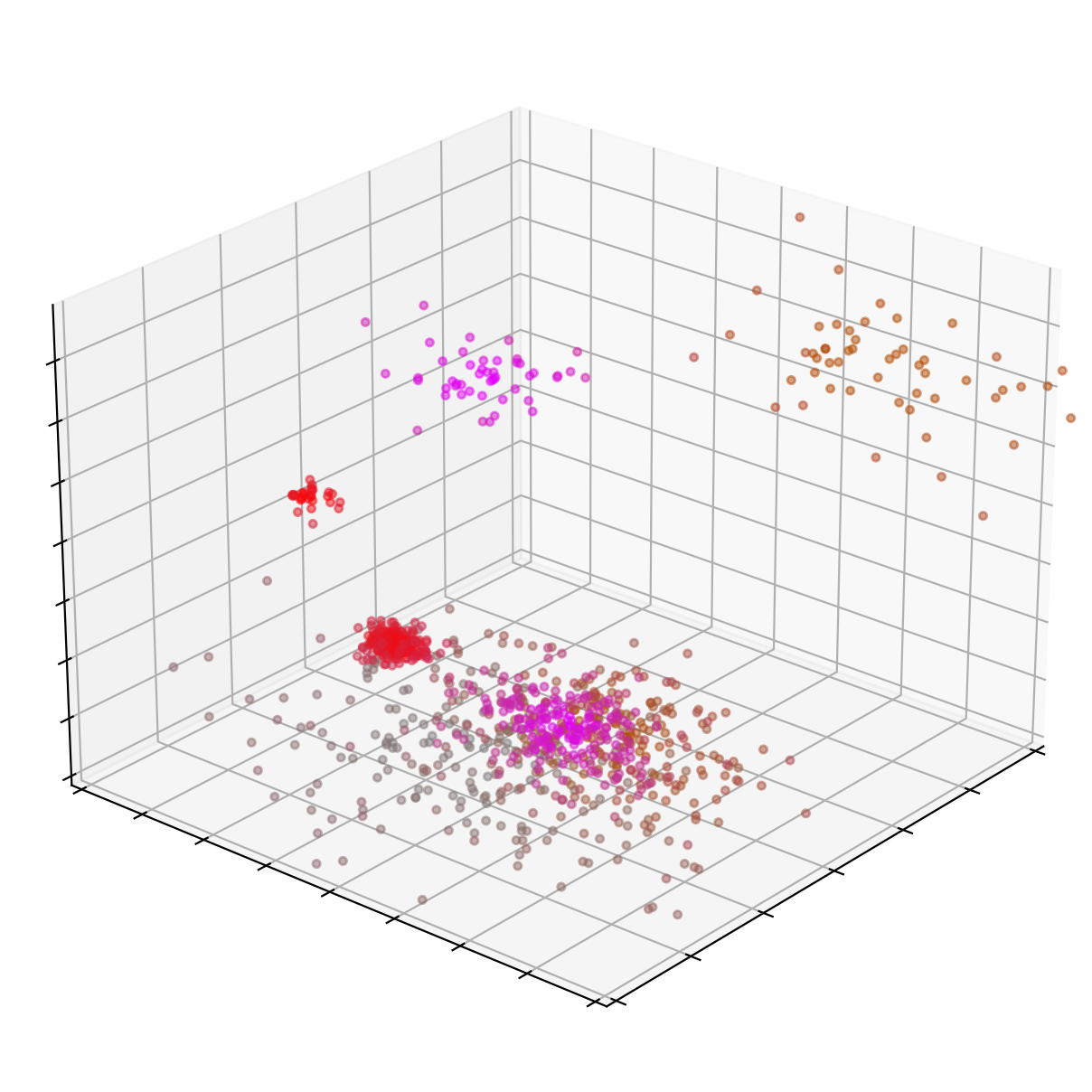}
	\includegraphics[width=0.3\linewidth]{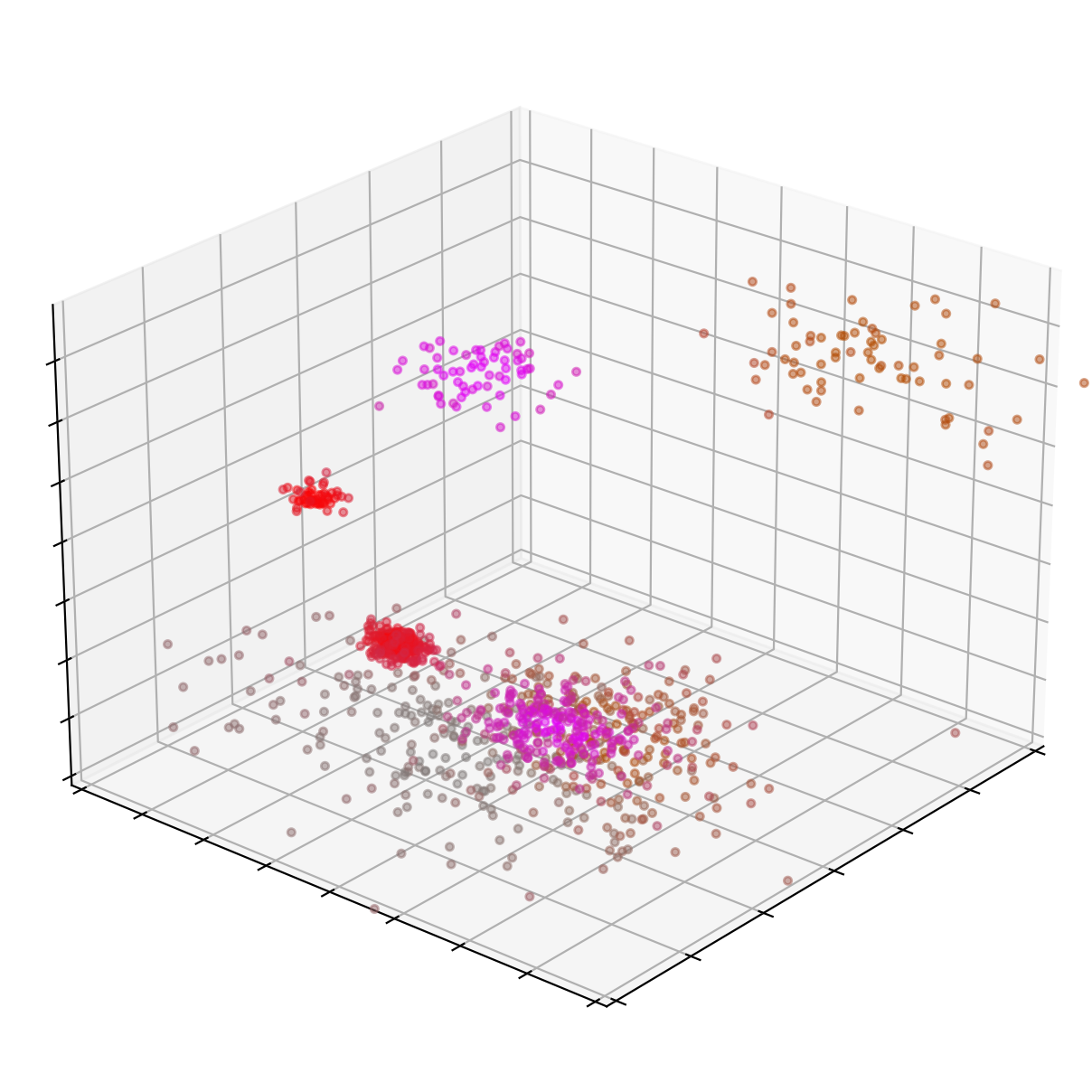}
	\includegraphics[width=0.3\linewidth]{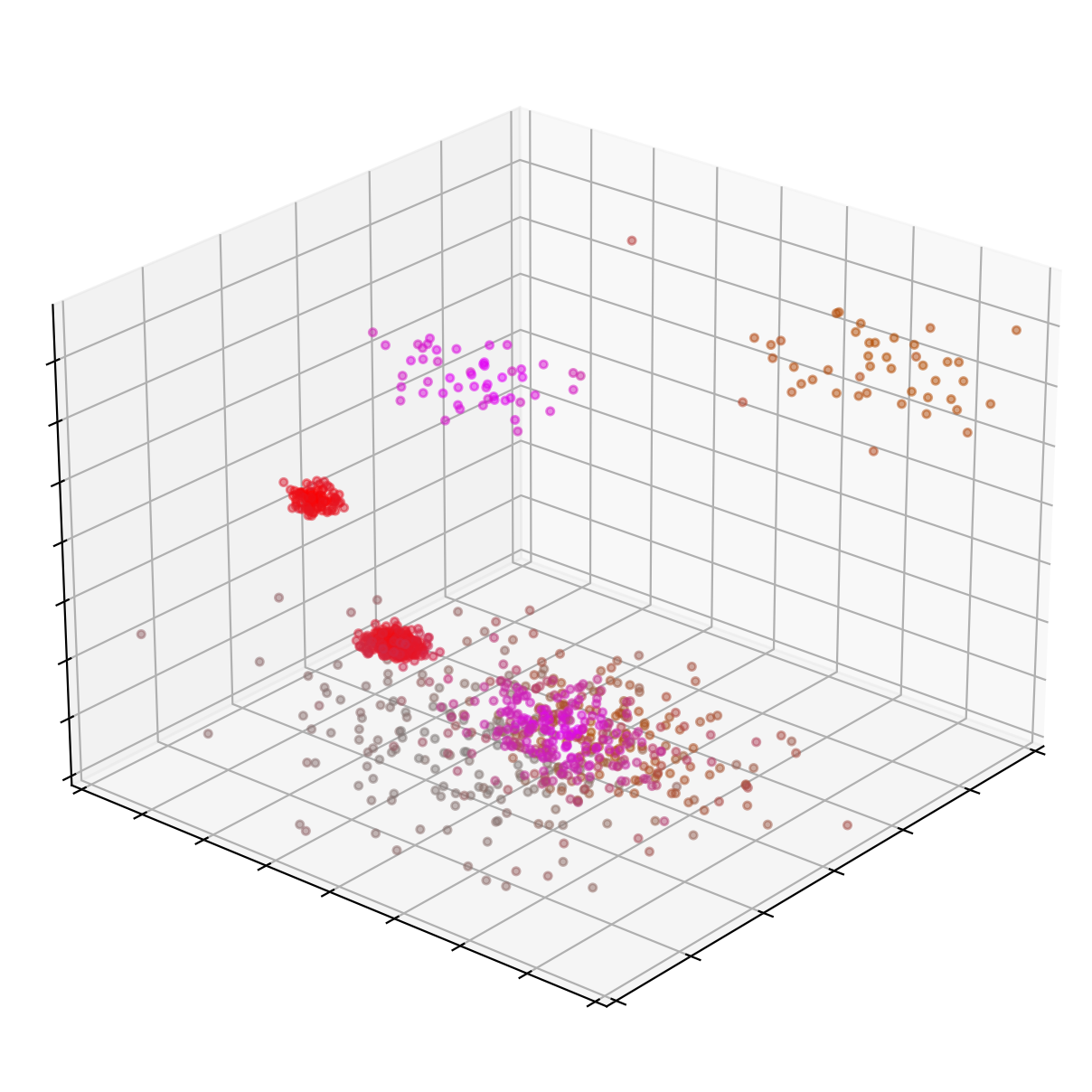}
	\caption{$3D$-scatter plots of iterations, represented as a $1000$-sized set of points, returned by the three methods \texttt{`hmc-StdG'} (top), \texttt{`hmc-RG'} (middle) and \texttt{`adhmc'} (bottom) corresponding to the algorithm settings corresponding to~\tbfg~\ref{fig:helix_hmcs} (right). Each particle is initiated at the origin. Compare to the sample set generated from the true distribution $\distribution$ presented in ~\tbfg~\ref{fig:helix_hmcs} (left).}
	\label{fig:iteratestrip}
\end{figure}

Each replication in the experiments in \tbfg~\ref{fig:helix_hmcs}(right) runs $|h_n|=1000$ chains concurrently. While in the \texttt{`hmc-StdG'} method, these chains remain independent of each other, the two adaptive schemes use the $n_C=2000$ last-visited values of these chains in modifying the auxiliary distribution resulting in cross-correlations between the chains. \tbfg~\ref{fig:adhmc_finetune} (left) presents the impact of the number $|h_n|$ of concurrent chains on the performance of the \texttt{`adhmc'} method for fixed integrator parameters $\epsilon=0.05$ and $L=200$; for clarity, we present only the average performance over the $50$ replications. Since the clustering-based auxiliary modification heuristic uses $n_C=2000$ visited samples, it is clear that the larger the number of concurrent chains provides a better estimate of the $\distribution(q)$ and hence allows for the most consistent convergence for $|h_n|=500$. 

\tbfg~\ref{fig:adhmc_finetune} (right) explores the impact of the approximations provided by the leapfrog symplectic integrator on \texttt{`adhmc'}'s convergence. Each of the three sets of experiments implements a total Hamiltonian motion length $T=5$, but with differing step size $\epsilon$ and number of steps $L$. With a large $\epsilon$ and a low count $L$, the method initially drops quickly but seems to plateau at a higher distance to the target because of the error inherent in the numerical approximation, which is also visible from the increased variance in its replications. On the other hand, the smallest $\epsilon$ method takes significantly longer to run because of the larger $L$. The size of $\epsilon=0.025$ seems to produce an adequate balance between the numerical error and the needed computation effort.

\begin{figure}[tb]
\vskip -0.1in
	\centering
		\includegraphics[width=0.47\linewidth]
		{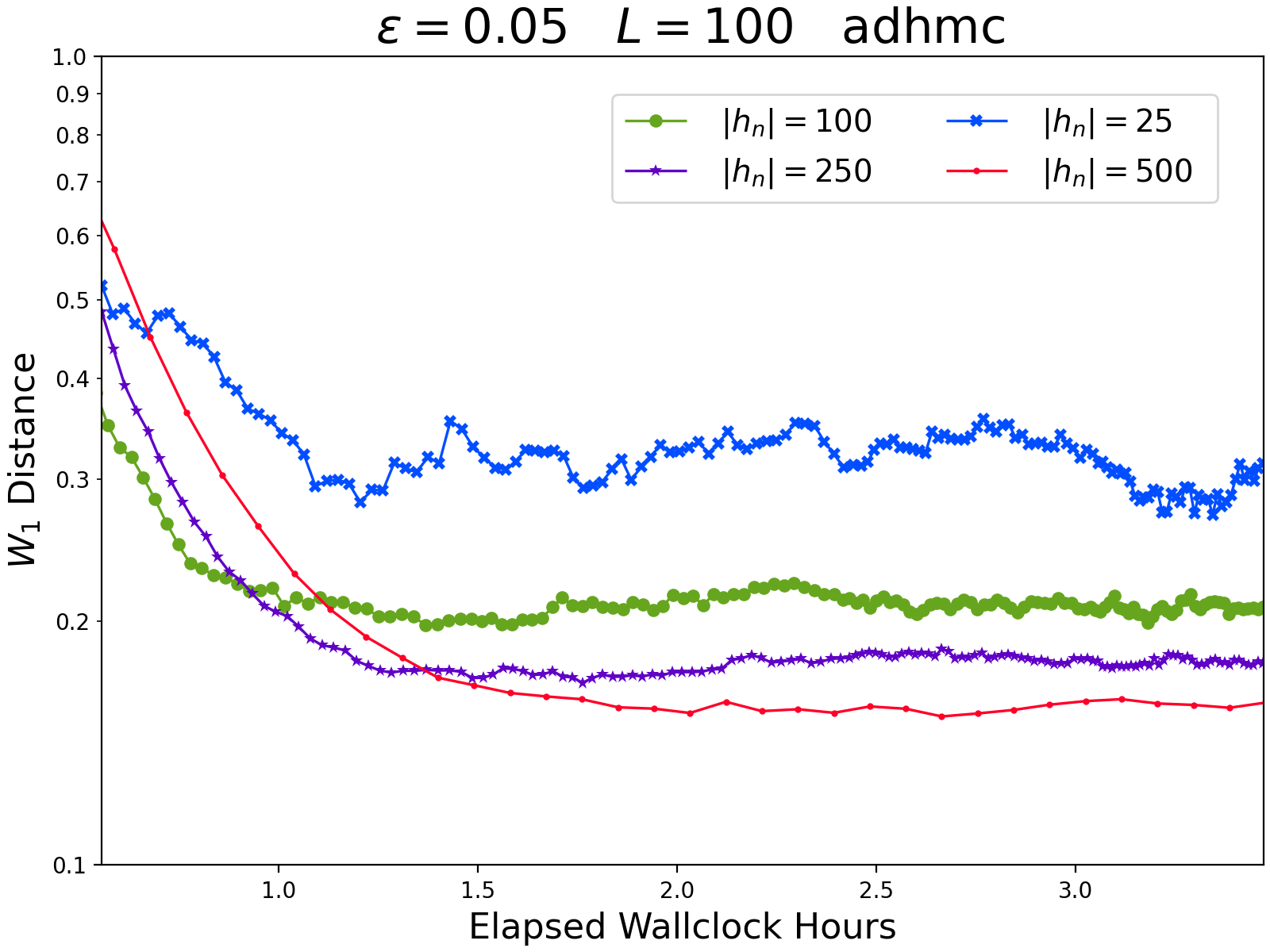}
\hskip 0.2in
	\includegraphics[width=0.47\linewidth]
	{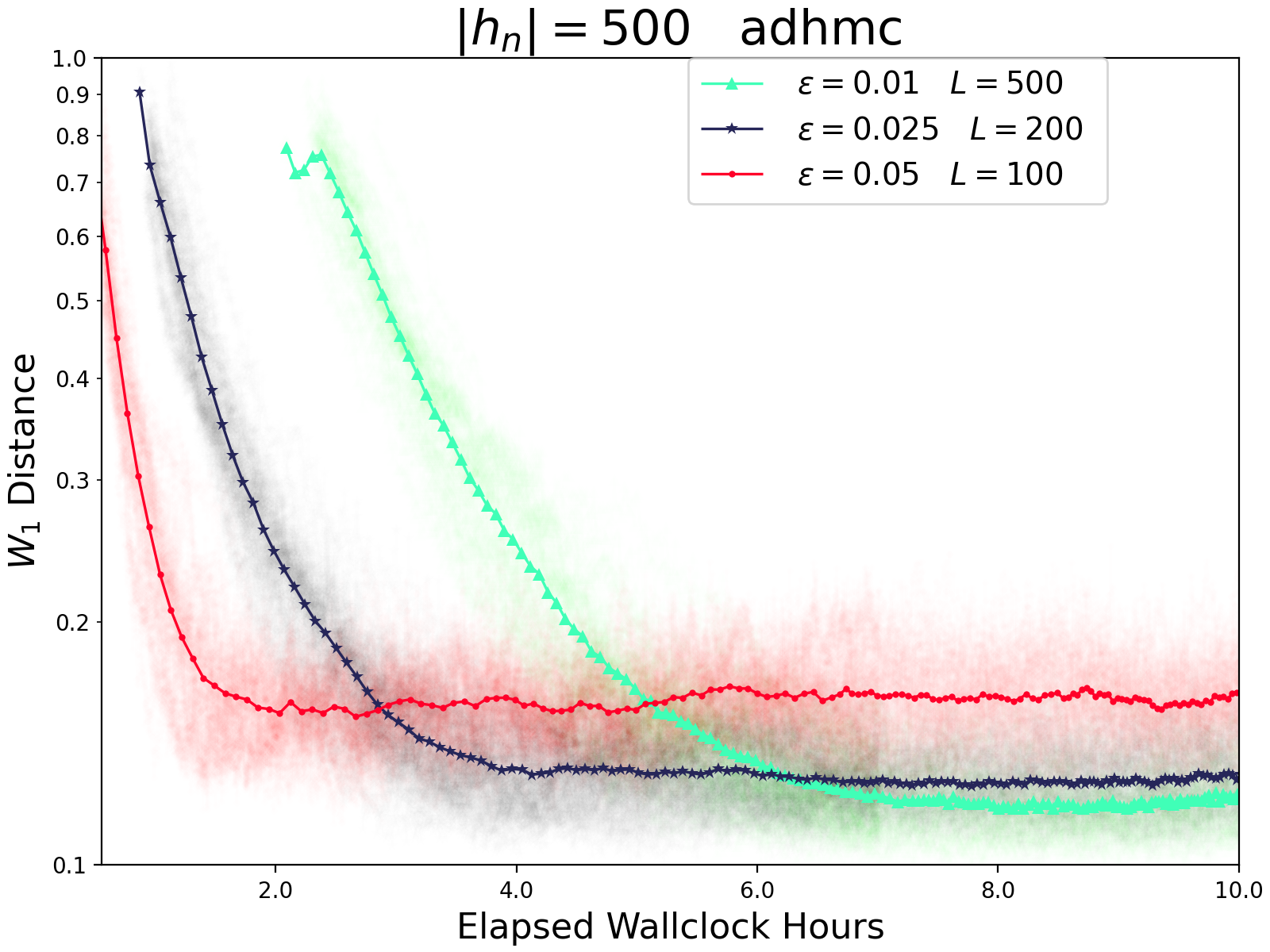}
\vskip -0.1in
	\caption{$W_1$ distance between iterate and  target distribution over total elapsed wall-clock hours for iterations of the AD-HMC method : (left) for fixed $\epsilon = 0.05, L=100$ and varying number of particles $|h|$, and (right) for particles $|h|=500$ and three settings that  implement total Hamiltonian motion of $T=\epsilon L = 5$ using the Leapfrog integrator}
	\label{fig:adhmc_finetune}
\end{figure}

Finally,~\tbfg~\ref{fig:helix_adhmcVhmc} presents a comparison of the 
the AD-HMC Algorithm~\ref{algo:adhmc} against  
the forward-motion only HMC Algorithm~\ref{algo:hmc} for a fixed auxiliary. On the left, an asymmetric mixture-of-Gaussians obtained from the \texttt{`adhmc'} scheme is set as the auxiliary for both algorithms, and each method uses $|h_n| = 1000$ concurrent chains and leapfrog parameters $\epsilon=0.05, L=200$. The need for the alternating direction motion steps in AD-HMC for asymmetric auxiliary distributions is clearly evident from these results, where there is no perceivable improvement in the $W_1$ distance from the target $\distribution(q)$ for the unidirectional-motion HMC method.  
Fig.~\ref{fig:helix_adhmcVhmc} (right) on the other hand starts both algorithms with the standard Gaussian as the auxiliary. The standard HMC (solid lines) with AD-HMC (dotted lines) algorithms use $|h_n| = 1000$ and leapfrog parameters (red) $L=20,\epsilon=0.01$, (green) $L=50,\epsilon=0.01$ and (blue) $L=100,\epsilon=0.01$. Both algorithms seem to work equally well for each of the settings displayed in the figure. 
This is to be expected since the forward and reverse motions of AD-HMC are stochastically identical to each step of the standard HMC, and thus both perform similarly over cumulative wallclock time elapsed. There is a slight improvement in convergence with higher leapfrog step counts for both algorithms.

%
\begin{figure}[tb]
	\centering
	\includegraphics[width=0.42\textwidth]
	{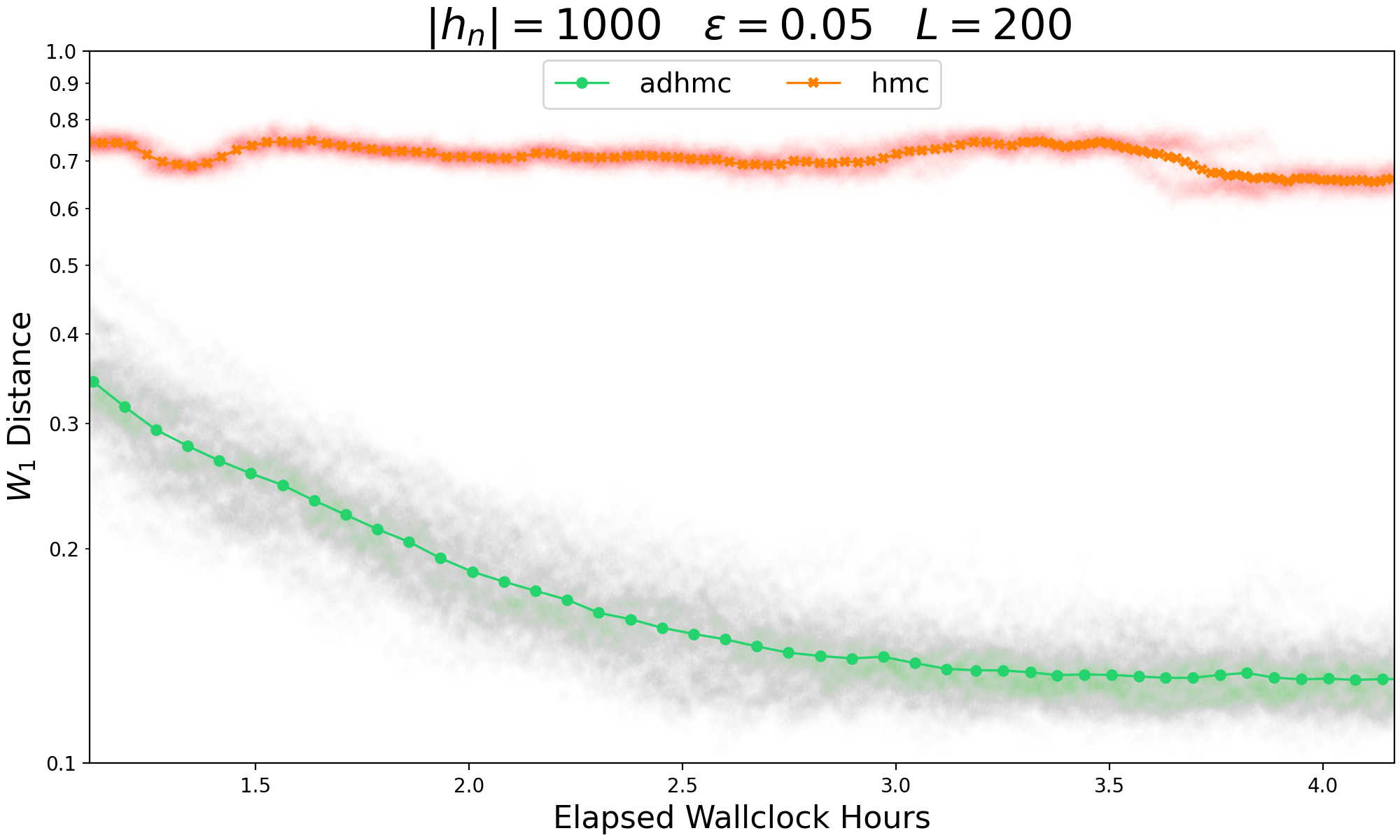}\hspace{0.25in}
\includegraphics[width=0.51\textwidth,trim={2cm 0 0 0},clip]{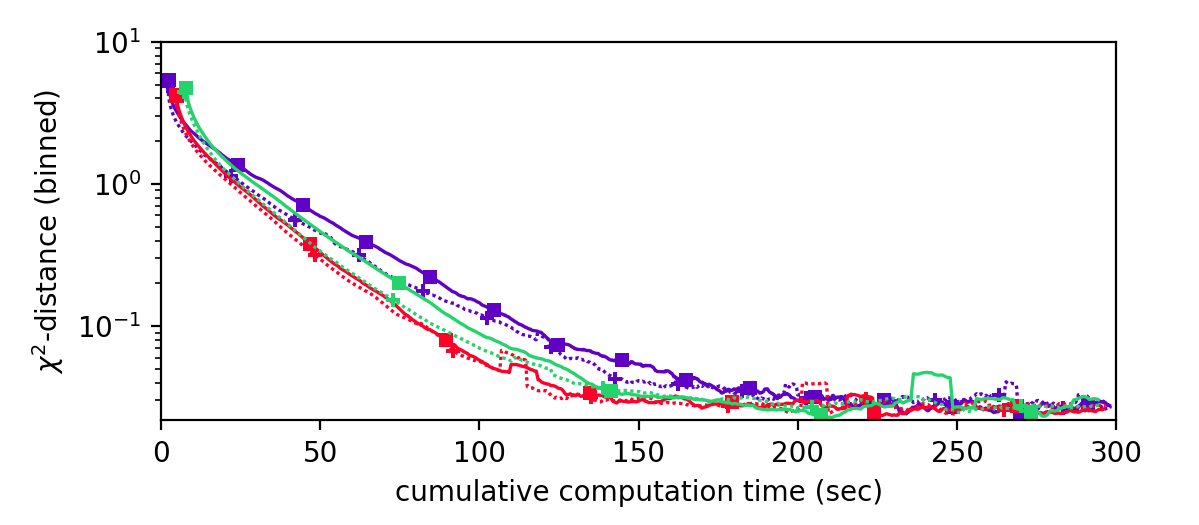}
	\captionof{figure}{Performance comparison of the forward-motion-only standard HMC and AD-HMC using $W_1$ distance between the iterate and the target distribution over iteration counts. On (Left), an asymmetric mixture-of-Gaussians auxiliary is used, with HMC in red and AD-HMC in green and the leapfrog parameters set to $L=200, \epsilon=0.05$. On (Right) a (symmetric) standard Gaussian auxiliary is used, and standard HMC is in solid lines and AD-HMC in dotted lines. Three colors represent (red) $L=20,\epsilon=0.01$, (green) $L=50,\epsilon=0.01$ and (blue) $L=100,\epsilon=0.01$.}
	\label{fig:helix_adhmcVhmc}
\end{figure}

\textbf{Summary.} The numerical evidence shows that the AD-HMC based \texttt{`adhmc'}, a generalization of the Riemannian-Gaussian procedure~\citep{GirolamiCalderhead11}, to model the auxiliary distribution on key characteristics of the target distribution, can notably speed up estimation of the target distribution over plain HMC \texttt{`hmc-StdG'} and adaptive HMC \texttt{`hmc-RG'}. The importance of alternating the direction when using asymmetric auxiliaries is also displayed. This shows a promising avenue for further exploration both in practical adaptive auxiliary design algorithms and in rigorously understanding when they may be guaranteed to converge faster.

\section{Conclusions}
\label{sec:conclusions}

We present a novel convergence analysis for
Hamiltonian Monte Carlo Algorithms with general momentum distributions. The analysis is deeply rooted in understanding of the dynamics of the density function, and it allows us to obtain results on geometric convergence of a large family of HMC algorithms, including AD-HMC, a novel algorithm proposed in this paper. In addition to demonstrating the effective of AD-HMC, our numerical studies also explore new possibilities in designing and refining HMC algorithms.

\bibliographystyle{elsarticle-harv}

\begin{thebibliography}{55}
\expandafter\ifx\csname natexlab\endcsname\relax\def\natexlab#1{#1}\fi
\providecommand{\url}[1]{\texttt{#1}}
\providecommand{\href}[2]{#2}
\providecommand{\path}[1]{#1}
\providecommand{\DOIprefix}{doi:}
\providecommand{\ArXivprefix}{arXiv:}
\providecommand{\URLprefix}{URL: }
\providecommand{\Pubmedprefix}{pmid:}
\providecommand{\doi}[1]{\href{http://dx.doi.org/#1}{\path{#1}}}
\providecommand{\Pubmed}[1]{\href{pmid:#1}{\path{#1}}}
\providecommand{\bibinfo}[2]{#2}
\ifx\xfnm\relax \def\xfnm[#1]{\unskip,\space#1}\fi
\bibitem[{Andrieu et~al.(2003)Andrieu, Freitas, Doucet and Jordan}]{andrieu03}
\bibinfo{author}{Andrieu, C.}, \bibinfo{author}{Freitas, N.D.}, \bibinfo{author}{Doucet, A.}, \bibinfo{author}{Jordan, M.I.}, \bibinfo{year}{2003}.
\newblock \bibinfo{title}{An introduction to mcmc for machine learning. machine learning}.
\newblock \bibinfo{journal}{Machine Learning} \bibinfo{volume}{50}, \bibinfo{pages}{5--–43}.
\bibitem[{Ankerst et~al.(1999)Ankerst, Breunig, Kriegel and Sander}]{opticscluster}
\bibinfo{author}{Ankerst, M.}, \bibinfo{author}{Breunig, M.M.}, \bibinfo{author}{Kriegel, H.P.}, \bibinfo{author}{Sander, J.}, \bibinfo{year}{1999}.
\newblock \bibinfo{title}{Optics: Ordering points to identify the clustering structure}.
\newblock \bibinfo{journal}{SIGMOD Rec.} \bibinfo{volume}{28}, \bibinfo{pages}{49–60}.
\newblock \URLprefix \url{https://doi.org/10.1145/304181.304187}, \DOIprefix\doi{10.1145/304181.304187}.
\bibitem[{Betancourt et~al.(2017)Betancourt, Byrne, Livingstone and Girolami}]{betancourt17}
\bibinfo{author}{Betancourt, M.}, \bibinfo{author}{Byrne, S.}, \bibinfo{author}{Livingstone, S.}, \bibinfo{author}{Girolami, M.}, \bibinfo{year}{2017}.
\newblock \bibinfo{title}{The geometric foundations of hamiltonian monte carlo}.
\newblock \bibinfo{journal}{Bernoulli} \bibinfo{volume}{23}, \bibinfo{pages}{2257--2298}.
\newblock \URLprefix \url{https://doi.org/10.3150/16-BEJ810}, \DOIprefix\doi{10.3150/16-BEJ810}.
\bibitem[{Bogachev(2007)}]{bogachev2007measure}
\bibinfo{author}{Bogachev, V.}, \bibinfo{year}{2007}.
\newblock \bibinfo{title}{Measure Theory}.
\newblock Number \bibinfo{number}{v. 1} in \bibinfo{series}{Measure Theory}, \bibinfo{publisher}{Springer Berlin Heidelberg}.
\newblock \URLprefix \url{https://books.google.com/books?id=CoSIe7h5mTsC}.
\bibitem[{Bou-Rabee et~al.(2020)Bou-Rabee, Eberle and Zimmer}]{bou-rabee2020}
\bibinfo{author}{Bou-Rabee, N.}, \bibinfo{author}{Eberle, A.}, \bibinfo{author}{Zimmer, R.}, \bibinfo{year}{2020}.
\newblock \bibinfo{title}{Coupling and convergence for {H}amiltonian {M}onte {C}arlo}.
\newblock \bibinfo{journal}{Ann. Appl. Probab.} \bibinfo{volume}{30}, \bibinfo{pages}{1209--1250}.
\newblock \URLprefix \url{https://doi.org/10.1214/19-AAP1528}, \DOIprefix\doi{10.1214/19-AAP1528}.
\bibitem[{Bou-Rabee and Sanz-Serna(2017)}]{bou-rabee2017}
\bibinfo{author}{Bou-Rabee, N.}, \bibinfo{author}{Sanz-Serna, J.M.}, \bibinfo{year}{2017}.
\newblock \bibinfo{title}{Randomized hamiltonian monte carlo}.
\newblock \bibinfo{journal}{Ann. Appl. Probab.} \bibinfo{volume}{27}, \bibinfo{pages}{2159--2194}.
\newblock \URLprefix \url{https://doi.org/10.1214/16-AAP1255}, \DOIprefix\doi{10.1214/16-AAP1255}.
\bibitem[{Brockwell and Kadane(2005)}]{brockwell05kadane}
\bibinfo{author}{Brockwell, A.E.}, \bibinfo{author}{Kadane, J.B.}, \bibinfo{year}{2005}.
\newblock \bibinfo{title}{Identification of regeneration times in mcmc simulation, with application to adaptive schemes}.
\newblock \bibinfo{journal}{Journal of Computational and Graphical Statistics} \bibinfo{volume}{14}, \bibinfo{pages}{436--458}.
\newblock \URLprefix \url{https://doi.org/10.1198/106186005X47453}, \DOIprefix\doi{10.1198/106186005X47453}, \href{http://arxiv.org/abs/https://doi.org/10.1198/106186005X47453}{{\tt arXiv:https://doi.org/10.1198/106186005X47453}}.
\bibitem[{Carpenter et~al.(2017)Carpenter, Gelman, Hoffman, Lee, Goodrich, Betancourt, Brubaker, Guo, Li and Riddell}]{carpenter17}
\bibinfo{author}{Carpenter, B.}, \bibinfo{author}{Gelman, A.}, \bibinfo{author}{Hoffman, M.}, \bibinfo{author}{Lee, D.}, \bibinfo{author}{Goodrich, B.}, \bibinfo{author}{Betancourt, M.}, \bibinfo{author}{Brubaker, M.}, \bibinfo{author}{Guo, J.}, \bibinfo{author}{Li, P.}, \bibinfo{author}{Riddell, A.}, \bibinfo{year}{2017}.
\newblock \bibinfo{title}{Stan: A probabilistic programming language}.
\newblock \bibinfo{journal}{Journal of Statistical Software, Articles} \bibinfo{volume}{76}, \bibinfo{pages}{1--32}.
\newblock \URLprefix \url{https://www.jstatsoft.org/v076/i01}, \DOIprefix\doi{10.18637/jss.v076.i01}.
\bibitem[{Chau and Rásonyi(2022)}]{chau22sghmc}
\bibinfo{author}{Chau, H.N.}, \bibinfo{author}{Rásonyi, M.}, \bibinfo{year}{2022}.
\newblock \bibinfo{title}{Stochastic gradient hamiltonian monte carlo for non-convex learning}.
\newblock \bibinfo{journal}{Stochastic Processes and their Applications} \bibinfo{volume}{149}, \bibinfo{pages}{341--368}.
\newblock \URLprefix \url{https://www.sciencedirect.com/science/article/pii/S0304414922000825}, \DOIprefix\doi{https://doi.org/10.1016/j.spa.2022.04.001}.
\bibitem[{Chen(2000)}]{CHEN2000281}
\bibinfo{author}{Chen, M.F.}, \bibinfo{year}{2000}.
\newblock \bibinfo{title}{Equivalence of exponential ergodicity and l2-exponential convergence for markov chains}.
\newblock \bibinfo{journal}{Stochastic Processes and their Applications} \bibinfo{volume}{87}, \bibinfo{pages}{281 -- 297}.
\newblock \URLprefix \url{http://www.sciencedirect.com/science/article/pii/S0304414999001143}, \DOIprefix\doi{https://doi.org/10.1016/S0304-4149(99)00114-3}.
\bibitem[{Chen and Vempala(2019)}]{chenvempala2019}
\bibinfo{author}{Chen, Z.}, \bibinfo{author}{Vempala, S.S.}, \bibinfo{year}{2019}.
\newblock \bibinfo{title}{Optimal convergence rate of hamiltonian monte carlo for strongly logconcave distributions}.
\newblock \bibinfo{journal}{RANDOM} .
\bibitem[{{de Lima} et~al.(2023){de Lima}, Corso, Ferreira and {de Ara\'ujo}}]{DELIMA2023128618}
\bibinfo{author}{{de Lima}, P.D.S.}, \bibinfo{author}{Corso, G.}, \bibinfo{author}{Ferreira, M.S.}, \bibinfo{author}{{de Ara\'ujo}, J.M.}, \bibinfo{year}{2023}.
\newblock \bibinfo{title}{Acoustic full waveform inversion with hamiltonian monte carlo method}.
\newblock \bibinfo{journal}{Physica A: Statistical Mechanics and its Applications} \bibinfo{volume}{617}, \bibinfo{pages}{128618}.
\newblock \URLprefix \url{https://www.sciencedirect.com/science/article/pii/S0378437123001735}, \DOIprefix\doi{https://doi.org/10.1016/j.physa.2023.128618}.
\bibitem[{Dhabaria and Singh(2024)}]{10.1093/gji/ggae112}
\bibinfo{author}{Dhabaria, N.}, \bibinfo{author}{Singh, S.C.}, \bibinfo{year}{2024}.
\newblock \bibinfo{title}{Hamiltonian monte carlo based elastic full-waveform inversion of wide-angle seismic data}.
\newblock \bibinfo{journal}{Geophysical Journal International} \bibinfo{volume}{237}, \bibinfo{pages}{1384--1399}.
\newblock \URLprefix \url{https://doi.org/10.1093/gji/ggae112}, \DOIprefix\doi{10.1093/gji/ggae112}, \href{http://arxiv.org/abs/https://academic.oup.com/gji/article-pdf/237/3/1384/57210592/ggae112.pdf}{{\tt arXiv:https://academic.oup.com/gji/article-pdf/237/3/1384/57210592/ggae112.pdf}}.
\bibitem[{Di~Sciullo(2009)}]{NLPMaria}
\bibinfo{author}{Di~Sciullo, A.M.}, \bibinfo{year}{2009}.
\newblock \bibinfo{title}{Natural language understanding}, in: \bibinfo{booktitle}{Proceedings of the Eighth SoMeT}, pp. \bibinfo{pages}{551--563}.
\newblock \DOIprefix\doi{10.3233/978-1-60750-049-0-551}.
\bibitem[{Duane et~al.(1987)Duane, Kennedy, Pendleton and Roweth}]{DuaneEtAl87}
\bibinfo{author}{Duane, S.}, \bibinfo{author}{Kennedy, A.}, \bibinfo{author}{Pendleton, B.J.}, \bibinfo{author}{Roweth, D.}, \bibinfo{year}{1987}.
\newblock \bibinfo{title}{Hybrid {M}onte {C}arlo}.
\newblock \bibinfo{journal}{Physics Letters B} \bibinfo{volume}{195}, \bibinfo{pages}{216 -- 222}.
\newblock \URLprefix \url{http://www.sciencedirect.com/science/article/pii/037026938791197X}, \DOIprefix\doi{https://doi.org/10.1016/0370-2693(87)91197-X}.
\bibitem[{Durmus and Moulines(2015)}]{Durmus2015}
\bibinfo{author}{Durmus, A.}, \bibinfo{author}{Moulines, {\'E}.}, \bibinfo{year}{2015}.
\newblock \bibinfo{title}{Quantitative bounds of convergence for geometrically ergodic markov chain in the wasserstein distance with application to the metropolis adjusted langevin algorithm}.
\newblock \bibinfo{journal}{Statistics and Computing} \bibinfo{volume}{25}, \bibinfo{pages}{5--19}.
\newblock \URLprefix \url{https://doi.org/10.1007/s11222-014-9511-z}, \DOIprefix\doi{10.1007/s11222-014-9511-z}.
\bibitem[{Durmus et~al.(2017)Durmus, Moulines and Saksman}]{Durmus2017OnTC}
\bibinfo{author}{Durmus, A.}, \bibinfo{author}{Moulines, {\'E}.}, \bibinfo{author}{Saksman, E.}, \bibinfo{year}{2017}.
\newblock \bibinfo{title}{On the convergence of hamiltonian monte carlo}.
\newblock \bibinfo{journal}{arXiv: Computation} .
\bibitem[{Durrett(2019)}]{durrett2019probability}
\bibinfo{author}{Durrett, R.}, \bibinfo{year}{2019}.
\newblock \bibinfo{title}{Probability: Theory and Examples}.
\newblock Cambridge Series in Statistical and Probabilistic Mathematics, \bibinfo{publisher}{Cambridge University Press}.
\newblock \URLprefix \url{https://books.google.com/books?id=vESPDwAAQBAJ}.
\bibitem[{Feydy et~al.(2019)Feydy, S{\'e}journ{\'e}, Vialard, Amari, Trouve and Peyr{\'e}}]{feydy2019interpolating}
\bibinfo{author}{Feydy, J.}, \bibinfo{author}{S{\'e}journ{\'e}, T.}, \bibinfo{author}{Vialard, F.X.}, \bibinfo{author}{Amari, S.i.}, \bibinfo{author}{Trouve, A.}, \bibinfo{author}{Peyr{\'e}, G.}, \bibinfo{year}{2019}.
\newblock \bibinfo{title}{Interpolating between optimal transport and mmd using sinkhorn divergences}, in: \bibinfo{booktitle}{The 22nd International Conference on Artificial Intelligence and Statistics}, pp. \bibinfo{pages}{2681--2690}.
\bibitem[{Fichtner et~al.(2018)Fichtner, Zunino and Gebraad}]{10.1093/gji/ggy496}
\bibinfo{author}{Fichtner, A.}, \bibinfo{author}{Zunino, A.}, \bibinfo{author}{Gebraad, L.}, \bibinfo{year}{2018}.
\newblock \bibinfo{title}{Hamiltonian monte carlo solution of tomographic inverse problems}.
\newblock \bibinfo{journal}{Geophysical Journal International} \bibinfo{volume}{216}, \bibinfo{pages}{1344--1363}.
\newblock \URLprefix \url{https://doi.org/10.1093/gji/ggy496}, \DOIprefix\doi{10.1093/gji/ggy496}, \href{http://arxiv.org/abs/https://academic.oup.com/gji/article-pdf/216/2/1344/41325089/gji\_216\_2\_1344.pdf}{{\tt arXiv:https://academic.oup.com/gji/article-pdf/216/2/1344/41325089/gji\_216\_2\_1344.pdf}}.
\bibitem[{Gao et~al.(2021)Gao, G\"urb\"uzbalaban and Zhu}]{gao21sghmc}
\bibinfo{author}{Gao, X.}, \bibinfo{author}{G\"urb\"uzbalaban, M.}, \bibinfo{author}{Zhu, L.}, \bibinfo{year}{2021}.
\newblock \bibinfo{title}{Global convergence of stochastic gradient hamiltonian monte carlo for nonconvex stochastic optimization: Nonasymptotic performance bounds and momentum-based acceleration}.
\newblock \bibinfo{journal}{Operations Research} \bibinfo{volume}{70}, \bibinfo{pages}{2931--2947}.
\bibitem[{Gebraad et~al.(2020)Gebraad, Boehm and Fichtner}]{https://doi.org/10.1029/2019JB018428}
\bibinfo{author}{Gebraad, L.}, \bibinfo{author}{Boehm, C.}, \bibinfo{author}{Fichtner, A.}, \bibinfo{year}{2020}.
\newblock \bibinfo{title}{Bayesian elastic full-waveform inversion using hamiltonian monte carlo}.
\newblock \bibinfo{journal}{Journal of Geophysical Research: Solid Earth} \bibinfo{volume}{125}, \bibinfo{pages}{e2019JB018428}.
\newblock \URLprefix \url{https://agupubs.onlinelibrary.wiley.com/doi/abs/10.1029/2019JB018428}, \DOIprefix\doi{https://doi.org/10.1029/2019JB018428}, \href{http://arxiv.org/abs/https://agupubs.onlinelibrary.wiley.com/doi/pdf/10.1029/2019JB018428}{{\tt arXiv:https://agupubs.onlinelibrary.wiley.com/doi/pdf/10.1029/2019JB018428}}. \bibinfo{note}{e2019JB018428 10.1029/2019JB018428}.
\bibitem[{Gelfand and Sahu(1994)}]{gelfand94sahu}
\bibinfo{author}{Gelfand, A.E.}, \bibinfo{author}{Sahu, S.K.}, \bibinfo{year}{1994}.
\newblock \bibinfo{title}{On markov chain monte carlo acceleration}.
\newblock \bibinfo{journal}{Journal of Computational and Graphical Statistics} \bibinfo{volume}{3}, \bibinfo{pages}{261--276}.
\newblock \URLprefix \url{http://www.jstor.org/stable/1390911}.
\bibitem[{Gelfand and Smith(1990)}]{gelfand90}
\bibinfo{author}{Gelfand, A.E.}, \bibinfo{author}{Smith, A.F.M.}, \bibinfo{year}{1990}.
\newblock \bibinfo{title}{Sampling-based approaches to calculating marginal densities}.
\newblock \bibinfo{journal}{Journal of the American Statistical Association} \bibinfo{volume}{85}, \bibinfo{pages}{398--409}.
\newblock \DOIprefix\doi{10.1080/01621459.1990.10476213}.
\bibitem[{Gelman et~al.(2013)Gelman, Carlin, Stern, Dunson, Vehtari and Rubin}]{gelman14}
\bibinfo{author}{Gelman, A.}, \bibinfo{author}{Carlin, J.B.}, \bibinfo{author}{Stern, H.S.}, \bibinfo{author}{Dunson, D.B.}, \bibinfo{author}{Vehtari, A.}, \bibinfo{author}{Rubin, D.B.}, \bibinfo{year}{2013}.
\newblock \bibinfo{title}{Bayesian Data Analysis}.
\newblock \bibinfo{publisher}{Chapman and Hall/CRC}.
\bibitem[{Ghosh et~al.(2022)Ghosh, Lu and Nowicki}]{GHOSH2022107811}
\bibinfo{author}{Ghosh, S.}, \bibinfo{author}{Lu, Y.}, \bibinfo{author}{Nowicki, T.}, \bibinfo{year}{2022}.
\newblock \bibinfo{title}{On ${L}_2$ convergence of the {H}amiltonian {M}onte {C}arlo}.
\newblock \bibinfo{journal}{Applied Mathematics Letters} \bibinfo{volume}{127}, \bibinfo{pages}{107811}.
\newblock \URLprefix \url{https://www.sciencedirect.com/science/article/pii/S0893965921004377}, \DOIprefix\doi{https://doi.org/10.1016/j.aml.2021.107811}.
\bibitem[{Gilks et~al.(1998)Gilks, Roberts and Sahu}]{gilksRoberts98Sahu}
\bibinfo{author}{Gilks, W.R.}, \bibinfo{author}{Roberts, G.O.}, \bibinfo{author}{Sahu, S.K.}, \bibinfo{year}{1998}.
\newblock \bibinfo{title}{Adaptive markov chain monte carlo through regeneration}.
\newblock \bibinfo{journal}{Journal of the American Statistical Association} \bibinfo{volume}{93}, \bibinfo{pages}{1045--1054}.
\newblock \URLprefix \url{https://doi.org/10.1080/01621459.1998.10473766}, \DOIprefix\doi{10.1080/01621459.1998.10473766}, \href{http://arxiv.org/abs/https://doi.org/10.1080/01621459.1998.10473766}{{\tt arXiv:https://doi.org/10.1080/01621459.1998.10473766}}.
\bibitem[{Girolami and Calderhead(2011)}]{GirolamiCalderhead11}
\bibinfo{author}{Girolami, M.}, \bibinfo{author}{Calderhead, B.}, \bibinfo{year}{2011}.
\newblock \bibinfo{title}{Riemann manifold langevin and hamiltonian monte carlo methods}.
\newblock \bibinfo{journal}{Journal of the Royal Statistical Society: Series B (Statistical Methodology)} \bibinfo{volume}{73}, \bibinfo{pages}{123--214}.
\newblock \URLprefix \url{https://rss.onlinelibrary.wiley.com/doi/abs/10.1111/j.1467-9868.2010.00765.x}, \DOIprefix\doi{10.1111/j.1467-9868.2010.00765.x}.
\bibitem[{Hairer et~al.(2013)Hairer, Lubich and Wanner}]{hairer2013geometric}
\bibinfo{author}{Hairer, E.}, \bibinfo{author}{Lubich, C.}, \bibinfo{author}{Wanner, G.}, \bibinfo{year}{2013}.
\newblock \bibinfo{title}{Geometric Numerical Integration: Structure-Preserving Algorithms for Ordinary Differential Equations}.
\newblock Springer Series in Computational Mathematics, \bibinfo{publisher}{Springer Berlin Heidelberg}.
\newblock \URLprefix \url{https://books.google.com/books?id=cPTxCAAAQBAJ}.
\bibitem[{Hairer et~al.(2011)Hairer, Mattingly and Scheutzow}]{HairerEtAl11}
\bibinfo{author}{Hairer, M.}, \bibinfo{author}{Mattingly, J.C.}, \bibinfo{author}{Scheutzow, M.}, \bibinfo{year}{2011}.
\newblock \bibinfo{title}{Asymptotic coupling and a general form of harris'theorem with applications to stochastic delay equations}.
\newblock \bibinfo{journal}{Probability Theory and Related Fields} \bibinfo{volume}{149}, \bibinfo{pages}{223--259}.
\newblock \URLprefix \url{https://doi.org/10.1007/s00440-009-0250-6}, \DOIprefix\doi{10.1007/s00440-009-0250-6}.
\bibitem[{Hastings(1970)}]{hastings70}
\bibinfo{author}{Hastings, W.K.}, \bibinfo{year}{1970}.
\newblock \bibinfo{title}{Monte carlo sampling methods using markov chains and their applications}.
\newblock \bibinfo{journal}{Biometrika} \bibinfo{volume}{57}, \bibinfo{pages}{97--109}.
\newblock \URLprefix \url{http://www.jstor.org/stable/2334940}.
\bibitem[{Hoffman and Gelman(2014)}]{hoffman14}
\bibinfo{author}{Hoffman, M.D.}, \bibinfo{author}{Gelman, A.}, \bibinfo{year}{2014}.
\newblock \bibinfo{title}{The no-u-turn sampler: Adaptively setting path lengths in hamiltonian monte carlo}.
\newblock \bibinfo{journal}{Journal of Machine Learning Research} \bibinfo{volume}{15}, \bibinfo{pages}{1593--1623}.
\newblock \URLprefix \url{http://jmlr.org/papers/v15/hoffman14a.html}.
\bibitem[{Jasche and Kitaura(2010)}]{jasche10}
\bibinfo{author}{Jasche, J.}, \bibinfo{author}{Kitaura, F.S.}, \bibinfo{year}{2010}.
\newblock \bibinfo{title}{Fast hamiltonian sampling for large‐scale structure inference}.
\newblock \bibinfo{journal}{Monthly Notices of the Royal Astronomical Society} \bibinfo{volume}{407}, \bibinfo{pages}{29 -- 42}.
\newblock \DOIprefix\doi{10.1111/j.1365-2966.2010.16897.x}.
\bibitem[{Joulin and Ollivier(2010)}]{joulin2010}
\bibinfo{author}{Joulin, A.}, \bibinfo{author}{Ollivier, Y.}, \bibinfo{year}{2010}.
\newblock \bibinfo{title}{Curvature, concentration and error estimates for markov chain monte carlo}.
\newblock \bibinfo{journal}{Ann. Probab.} \bibinfo{volume}{38}, \bibinfo{pages}{2418--2442}.
\newblock \URLprefix \url{https://doi.org/10.1214/10-AOP541}, \DOIprefix\doi{10.1214/10-AOP541}.
\bibitem[{Leimkuhler and Reich(2004)}]{leimkuhler04}
\bibinfo{author}{Leimkuhler, B.}, \bibinfo{author}{Reich, S.}, \bibinfo{year}{2004}.
\newblock \bibinfo{title}{Simulating Hamiltonian Dynamics}.
\newblock Cambridge Monographs on Applied and Computational Mathematics, \bibinfo{publisher}{Cambridge University Press}.
\newblock \URLprefix \url{https://books.google.com/books?id=tpb-tnsZi5YC}.
\bibitem[{Lelievre et~al.(2010)Lelievre, Rousset and Stoltz}]{lilievre10}
\bibinfo{author}{Lelievre, T.}, \bibinfo{author}{Rousset, M.}, \bibinfo{author}{Stoltz, G.}, \bibinfo{year}{2010}.
\newblock \bibinfo{title}{Langevin dynamics with constraints and computation of free energy differences}.
\newblock \href{http://arxiv.org/abs/arXiv:1006.4914}{{\tt arXiv:arXiv:1006.4914}}.
\bibitem[{Livingstone et~al.(2019)Livingstone, Betancourt, Byrne and Girolami}]{livingstone2019}
\bibinfo{author}{Livingstone, S.}, \bibinfo{author}{Betancourt, M.}, \bibinfo{author}{Byrne, S.}, \bibinfo{author}{Girolami, M.}, \bibinfo{year}{2019}.
\newblock \bibinfo{title}{On the geometric ergodicity of {H}amiltonian {M}onte {C}arlo}.
\newblock \bibinfo{journal}{Bernoulli} \bibinfo{volume}{25}, \bibinfo{pages}{3109--3138}.
\newblock \URLprefix \url{https://doi.org/10.3150/18-BEJ1083}, \DOIprefix\doi{10.3150/18-BEJ1083}.
\bibitem[{Lockwood et~al.(2024)Lockwood, Weiss, Aronshtein and Verdon}]{PhysRevResearch.6.033142}
\bibinfo{author}{Lockwood, O.}, \bibinfo{author}{Weiss, P.}, \bibinfo{author}{Aronshtein, F.}, \bibinfo{author}{Verdon, G.}, \bibinfo{year}{2024}.
\newblock \bibinfo{title}{Quantum dynamical hamiltonian monte carlo}.
\newblock \bibinfo{journal}{Phys. Rev. Res.} \bibinfo{volume}{6}, \bibinfo{pages}{033142}.
\newblock \URLprefix \url{https://link.aps.org/doi/10.1103/PhysRevResearch.6.033142}, \DOIprefix\doi{10.1103/PhysRevResearch.6.033142}.
\bibitem[{Mangoubi and Smith(2019)}]{rapidmixing}
\bibinfo{author}{Mangoubi, O.}, \bibinfo{author}{Smith, A.}, \bibinfo{year}{2019}.
\newblock \bibinfo{title}{Rapid mixing of hamiltonian monte carlo on strongly log-concave distributions}.
\newblock \bibinfo{journal}{Proceedings of Machine Learning Research} \bibinfo{volume}{89}.
\bibitem[{Markowich and Villani(2000)}]{MV2000}
\bibinfo{author}{Markowich, P.A.}, \bibinfo{author}{Villani, C.}, \bibinfo{year}{2000}.
\newblock \bibinfo{title}{On the trend to equilibrium for the fokker-planck equation: An interplay between physics and functional analysis}.
\newblock \bibinfo{journal}{Matematica Contemporanea (SBM)} \bibinfo{volume}{19}, \bibinfo{pages}{1--31}.
\bibitem[{Meyn and Tweedie(2009)}]{meyn2009markov}
\bibinfo{author}{Meyn, S.}, \bibinfo{author}{Tweedie, R.}, \bibinfo{year}{2009}.
\newblock \bibinfo{title}{Markov Chains and Stochastic Stability}.
\newblock Cambridge Mathematical Library, \bibinfo{publisher}{Cambridge University Press}.
\newblock \URLprefix \url{https://books.google.com/books?id=Md7RnYEPkJwC}.
\bibitem[{Neal(1993)}]{neal93}
\bibinfo{author}{Neal, R.M.}, \bibinfo{year}{1993}.
\newblock \bibinfo{title}{Bayesian learning via stochastic dynamics}, in: \bibinfo{editor}{Hanson, S.J.}, \bibinfo{editor}{Cowan, J.D.}, \bibinfo{editor}{Giles, C.L.} (Eds.), \bibinfo{booktitle}{Advances in Neural Information Processing Systems 5}. \bibinfo{publisher}{Morgan-Kaufmann}, pp. \bibinfo{pages}{475--482}.
\bibitem[{Ollivier(2009)}]{OLLIVIER2009810}
\bibinfo{author}{Ollivier, Y.}, \bibinfo{year}{2009}.
\newblock \bibinfo{title}{Ricci curvature of markov chains on metric spaces}.
\newblock \bibinfo{journal}{Journal of Functional Analysis} \bibinfo{volume}{256}, \bibinfo{pages}{810 -- 864}.
\newblock \URLprefix \url{http://www.sciencedirect.com/science/article/pii/S002212360800493X}, \DOIprefix\doi{https://doi.org/10.1016/j.jfa.2008.11.001}.
\bibitem[{Ostmeyer et~al.(2021)Ostmeyer, Berkowitz, Luu, Petschlies and Pittler}]{OSTMEYER2021107978}
\bibinfo{author}{Ostmeyer, J.}, \bibinfo{author}{Berkowitz, E.}, \bibinfo{author}{Luu, T.}, \bibinfo{author}{Petschlies, M.}, \bibinfo{author}{Pittler, F.}, \bibinfo{year}{2021}.
\newblock \bibinfo{title}{The ising model with hybrid monte carlo}.
\newblock \bibinfo{journal}{Computer Physics Communications} \bibinfo{volume}{265}, \bibinfo{pages}{107978}.
\newblock \URLprefix \url{https://www.sciencedirect.com/science/article/pii/S0010465521000904}, \DOIprefix\doi{https://doi.org/10.1016/j.cpc.2021.107978}.
\bibitem[{Pedregosa et~al.(2011)Pedregosa, Varoquaux, Gramfort, Michel, Thirion, Grisel, Blondel, Prettenhofer, Weiss, Dubourg, Vanderplas, Passos, Cournapeau, Brucher, Perrot and Duchesnay}]{scikit-learn}
\bibinfo{author}{Pedregosa, F.}, \bibinfo{author}{Varoquaux, G.}, \bibinfo{author}{Gramfort, A.}, \bibinfo{author}{Michel, V.}, \bibinfo{author}{Thirion, B.}, \bibinfo{author}{Grisel, O.}, \bibinfo{author}{Blondel, M.}, \bibinfo{author}{Prettenhofer, P.}, \bibinfo{author}{Weiss, R.}, \bibinfo{author}{Dubourg, V.}, \bibinfo{author}{Vanderplas, J.}, \bibinfo{author}{Passos, A.}, \bibinfo{author}{Cournapeau, D.}, \bibinfo{author}{Brucher, M.}, \bibinfo{author}{Perrot, M.}, \bibinfo{author}{Duchesnay, E.}, \bibinfo{year}{2011}.
\newblock \bibinfo{title}{Scikit-learn: Machine learning in {P}ython}.
\newblock \bibinfo{journal}{Journal of Machine Learning Research} \bibinfo{volume}{12}, \bibinfo{pages}{2825--2830}.
\bibitem[{Robert and Casella(2004)}]{robert04}
\bibinfo{author}{Robert, C.}, \bibinfo{author}{Casella, G.}, \bibinfo{year}{2004}.
\newblock \bibinfo{title}{Monte Carlo Statistical Methods}.
\newblock \bibinfo{publisher}{Springer}.
\bibitem[{Roberts and Rosenthal(2001)}]{doi:10.1081/STM-100002060}
\bibinfo{author}{Roberts, G.O.}, \bibinfo{author}{Rosenthal, J.S.}, \bibinfo{year}{2001}.
\newblock \bibinfo{title}{Small and pseudo-small sets for markov chains}.
\newblock \bibinfo{journal}{Stochastic Models} \bibinfo{volume}{17}, \bibinfo{pages}{121--145}.
\newblock \URLprefix \url{https://doi.org/10.1081/STM-100002060}, \DOIprefix\doi{10.1081/STM-100002060}, \href{http://arxiv.org/abs/https://doi.org/10.1081/STM-100002060}{{\tt arXiv:https://doi.org/10.1081/STM-100002060}}.
\bibitem[{Roberts and Rosenthal(2004)}]{roberts2004}
\bibinfo{author}{Roberts, G.O.}, \bibinfo{author}{Rosenthal, J.S.}, \bibinfo{year}{2004}.
\newblock \bibinfo{title}{General state space markov chains and mcmc algorithms}.
\newblock \bibinfo{journal}{Probab. Surveys} \bibinfo{volume}{1}, \bibinfo{pages}{20--71}.
\newblock \URLprefix \url{https://doi.org/10.1214/154957804100000024}, \DOIprefix\doi{10.1214/154957804100000024}.
\bibitem[{Rosenthal(2002)}]{rosenthal2002}
\bibinfo{author}{Rosenthal, J.}, \bibinfo{year}{2002}.
\newblock \bibinfo{title}{Quantitative convergence rates of markov chains: A simple account}.
\newblock \bibinfo{journal}{Electron. Commun. Probab.} \bibinfo{volume}{7}, \bibinfo{pages}{123--128}.
\newblock \URLprefix \url{https://doi.org/10.1214/ECP.v7-1054}, \DOIprefix\doi{10.1214/ECP.v7-1054}.
\bibitem[{Schork et~al.(1990)Schork, Weder, Schork and Rao}]{https://doi.org/10.1002/gepi.1370070605}
\bibinfo{author}{Schork, N.J.}, \bibinfo{author}{Weder, A.B.}, \bibinfo{author}{Schork, M.A.}, \bibinfo{author}{Rao, D.C.}, \bibinfo{year}{1990}.
\newblock \bibinfo{title}{On the asymmetry of biological frequency distributions}.
\newblock \bibinfo{journal}{Genetic Epidemiology} \bibinfo{volume}{7}, \bibinfo{pages}{427--446}.
\newblock \URLprefix \url{https://onlinelibrary.wiley.com/doi/abs/10.1002/gepi.1370070605}, \DOIprefix\doi{https://doi.org/10.1002/gepi.1370070605}, \href{http://arxiv.org/abs/https://onlinelibrary.wiley.com/doi/pdf/10.1002/gepi.1370070605}{{\tt arXiv:https://onlinelibrary.wiley.com/doi/pdf/10.1002/gepi.1370070605}}.
\bibitem[{Stuart(2010)}]{stuart10}
\bibinfo{author}{Stuart, A.M.}, \bibinfo{year}{2010}.
\newblock \bibinfo{title}{Inverse problems: A bayesian perspective}.
\newblock \bibinfo{journal}{Acta Numerica} \bibinfo{volume}{19}, \bibinfo{pages}{451–559}.
\newblock \DOIprefix\doi{10.1017/S0962492910000061}.
\bibitem[{Talagrand(1996)}]{Talagrand1996}
\bibinfo{author}{Talagrand, M.}, \bibinfo{year}{1996}.
\newblock \bibinfo{title}{Transportation cost for gaussian and other product measures}.
\newblock \bibinfo{journal}{Geometric {\&} Functional Analysis GAFA} \bibinfo{volume}{6}, \bibinfo{pages}{587--600}.
\newblock \URLprefix \url{https://doi.org/10.1007/BF02249265}, \DOIprefix\doi{10.1007/BF02249265}.
\bibitem[{Team(2017)}]{stan17}
\bibinfo{author}{Team, S.D.}, \bibinfo{year}{2017}.
\newblock \bibinfo{title}{Stan modeling language users guide and reference manual}.
\newblock \URLprefix \url{https://mc-stan.org/}.
\bibitem[{Verlet(1967)}]{verlet67}
\bibinfo{author}{Verlet, L.}, \bibinfo{year}{1967}.
\newblock \bibinfo{title}{Computer "experiments" on classical fluids. i. thermodynamical properties of lennard-jones molecules}.
\newblock \bibinfo{journal}{Phys. Rev.} \bibinfo{volume}{159}, \bibinfo{pages}{98--103}.
\newblock \URLprefix \url{https://link.aps.org/doi/10.1103/PhysRev.159.98}, \DOIprefix\doi{10.1103/PhysRev.159.98}.
\bibitem[{Villani(2008)}]{villani2008optimal}
\bibinfo{author}{Villani, C.}, \bibinfo{year}{2008}.
\newblock \bibinfo{title}{Optimal Transport: Old and New}.
\newblock Grundlehren der mathematischen Wissenschaften, \bibinfo{publisher}{Springer Berlin Heidelberg}.
\newblock \URLprefix \url{https://books.google.com/books?id=hV8o5R7\_5tkC}.

\end{thebibliography}

\end{document}